\PassOptionsToPackage{capitalise,noabbrev}{cleveref}

\documentclass[10pt,twocolumn,letterpaper]{article}

\usepackage{cvpr}              
\usepackage[accsupp]{axessibility}  

\usepackage{subcaption}
\usepackage{graphicx}
\usepackage{xcolor}
\definecolor{customGreen}{rgb}{0.0, 0.9, 0.0}  
\definecolor{customMagenta}{rgb}{0.9, 0.0, 0.9}  
\definecolor{customBlue}{rgb}{0.0, 0.6, 0.95}
\definecolor{customRed}{rgb}{0.8, 0.12, 0.0}
\usepackage{caption}

\usepackage{amsthm}

\newtheoremstyle{withname}
  {}{}            
  {\itshape}      
  {}              
  {\bfseries}     
  {.}             
  {.5em}          
  {\thmname{#1}\thmnumber{ #2}\thmnote{ (#3)}} 
\theoremstyle{withname}
\newtheorem{proposition}{Proposition}

\usepackage{amsfonts}
\usepackage{float}
\usepackage{bbm}

\definecolor{cvprblue}{rgb}{0.34,0.4,0.85} 

\usepackage[pagebackref,breaklinks,colorlinks,allcolors=cvprblue]{hyperref}

\usepackage{cuted}   
\usepackage{capt-of} 
\usepackage{caption} 
\usepackage{tabularx}
\newcolumntype{Y}{>{\centering\arraybackslash}X} 

\usepackage{dblfloatfix} 

\usepackage{siunitx,makecell,booktabs}
\usepackage{pifont}

\usepackage{multirow}
\usepackage{rotating}

\newlength{\tileh}
\newlength{\labw}

\usepackage[table]{xcolor}

\newcommand{\dataset}{VT-Intrinsic\ }

\usepackage{tikz}
\usetikzlibrary{shapes.geometric}
\newcommand{\optmark}{%
  \tikz[baseline=-0.6ex]\node[draw=blue!50,fill=blue!50,circle,
  minimum size=1.2ex,inner sep=0pt]{}; }
\newcommand{\learnmark}{%
  \tikz[baseline=-0.6ex]\node[draw=magenta!70,fill=magenta!70,diamond,
  minimum size=1.2ex,inner sep=0pt]{}; }
\newcommand{\physmark}{%
  \tikz[baseline=-0.6ex]\node[draw=teal!70,fill=teal!70,star,star points=5,
  star point ratio=2.3,minimum size=1.4ex,inner sep=0pt]{}; }

\usepackage{dblfloatfix}

\newcommand{\ycircle}{%
  \tikz[baseline=-0.6ex, every node/.style={inner sep=0, outer sep=0}]{
    \draw[draw={yellow!85!orange}, line width=1.8pt] (0,0) circle (0.8ex);
  }%
}

\definecolor{cvprblue}{rgb}{0.21,0.49,0.74}
\usepackage[pagebackref,breaklinks,colorlinks,allcolors=cvprblue]{hyperref}

\def\paperID{} 
\def\confName{CVPR}
\def\confYear{2026}

\title{
VT-Intrinsic: Physics-Based Decomposition of Reflectance and Shading\\ 
using a Single Visible-Thermal Image Pair
\vspace{-0.1in}
}

\author{Zeqing Yuan\quad 
Mani Ramanagopal\quad 
Aswin C. Sankaranarayanan\quad
Srinivasa G. Narasimhan\\
{Carnegie Mellon University}\\
{\tt\small\href{https://vt-intrinsic.github.io}{https://vt-intrinsic.github.io}}
\vspace{-0.1in}
}

\begin{document}

\maketitle

\begin{strip}
    \includegraphics[width=\linewidth]{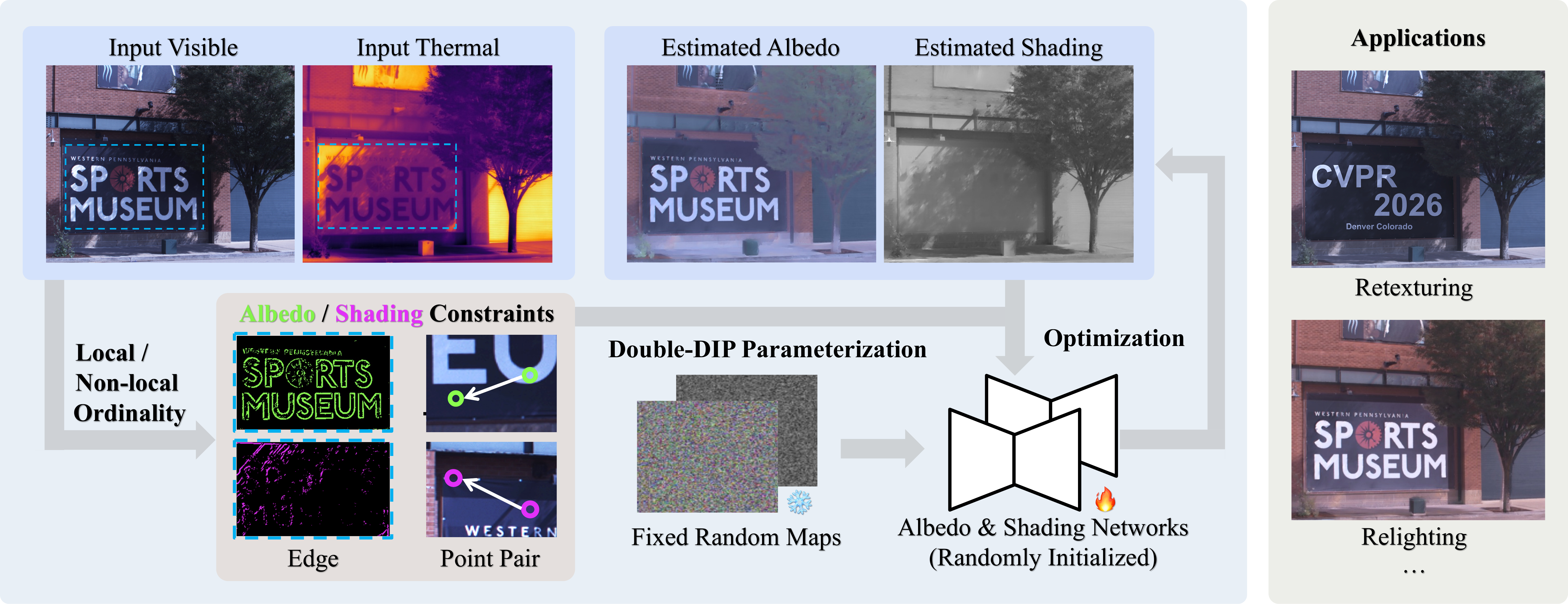}
    \vspace{-0.2in}
    \captionof{figure}{
        Through the tree's veil, sunlight weaves intricate shadows across a building façade.
        Visible and thermal images capture complementary cues of the reflected and absorbed light. 
        Local and non-local visible–thermal ordinalities (\autoref{sec:theory}) reveal albedo-/shading-dominant edges and point-pair ordinalities respectively, guiding an optimization using Double-DIP parameterization (\autoref{sec:method}).
        Our physics-based method reconstructs the complex shading and albedo without learned priors, whereas state-of-the-art models fail (see supplementary). 
        The decomposition enables faithful retexturing and relighting via albedo and shading editing. 
    }
    \label{fig:teaser_main}
\end{strip}

\begin{center}
{\large \textbf{Abstract}}
\end{center}

Decomposing a scene into its reflectance and shading is a challenge due to the lack of extensive ground-truth data for real-world scenes. We introduce a novel physics-based approach for intrinsic image decomposition using a pair of visible and thermal images. We leverage the principle that light not reflected from an opaque surface is absorbed and detected as heat by a thermal camera. This allows us to relate the ordinalities (or relative magnitudes) between visible and thermal image intensities to the ordinalities of shading and reflectance. The ordinalities enable dense self-supervision of an optimizing neural network to recover shading and reflectance. We perform quantitative evaluations with known reflectance and shading under natural and artificial lighting, and qualitative experiments across diverse scenes. The results demonstrate superior performance over both physics-based and recent learning-based methods, providing a path toward scalable real-world data curation with supervision. 

\section{Introduction}
\label{sec:intro}


Understanding how a scene appears from the interaction between \textit{surface reflectance} (the intrinsic material color, often approximated as diffuse \textit{albedo}) and \textit{incident illumination} (\textit{shading} determined by lighting and geometry) has long been a central pursuit in vision and imaging sciences~\cite{barrow1978IID}.
Disentangling these physical factors is useful for various applications in  graphics (recoloring, relighting, and compositing) and vision (object recognition and tracking).
Recent learning-based methods have made progress by formulating this task in an end-to-end framework and inferring statistical priors from auxiliary datasets to constrain the otherwise ill-posed inverse problem~\cite{garces2022survey}.
%
However, collecting ground-truth data for real-world scenes remains infeasible, as measuring surface reflectance and shading requires specialized equipment and controlled procedures~\cite{MITIntrinsic,wu2023maw}.

In this paper, we introduce a novel physics-based framework that leverages a single auxiliary thermal image to decompose a visible image of a scene into its albedo and shading components. 
To see why a thermal image is useful here, we consider the underlying physical principles that govern albedo and shading. 
Shading corresponds to the total incident energy (or irradiance) at a scene point, while albedo represents the proportion of that energy reflected by the surface. 
%
For opaque objects, the unreflected portion of the incident energy is absorbed, which contributes to the thermal radiation. This radiation can be detected by a thermal camera in the long-wave infrared range (8–14 $\mu\mathrm{m}$).
%
%
A recent technique called JoLHT-Video~\cite{joint_light_and_heat_transport} addressed this issue by modeling heat transport equations and estimating it from the transient heat flow observed by a thermal \textit{video}. 
However, directly estimating the absorbed light is challenging as it requires controlled and visible-only lighting and transient thermal measurements from video,
Inspired by this work, we pose the following question: {\it What can be achieved using only a single thermal image?}

Since absorption of light increases the temperature of an object, low-albedo regions—dark in the visible image—appear bright in the thermal image, whereas shading variations appear bright in both.
%
Based on this observation, we relate visible–thermal intensity ordinalities between any two scene points to their albedo and shading ordinalities, {\it without} having to estimate the absorbed light. 
%
Specifically, the ordinality of neighboring scene points classifies edges as shading- or reflectance-dominant and defines an edge loss, while non-local ordinalities yield a point-pair loss.
These new losses are used, together with the standard visible-image reconstruction loss, to optimize a neural network (e.g., Double Deep Image Prior~\cite{gandelsman2018doubledip}, randomly initialized to parameterize albedo and shading), effectively providing dense self-supervision for intrinsic decomposition.
%

Our ordinality theory is derived using the Lambertian assumption and when illumination is confined to the visible spectrum (e.g., LED lighting).
%
We further extend it to broadband sources containing infrared energy (e.g., sunlight, incandescent bulbs) by empirically observing—and statistically validating with~\cite{kokaly2017usgs}—that infrared albedo exhibits lower variation across common materials than visible albedo~\cite{ChoeG}, thereby preserving ordinalities. 
%
Expert validation on diverse materials and natural scenes—including those moderately violating the Lambertian assumption—shows near-perfect agreement between our automatically estimated point-pair ordinalities and confident expert labels, confirming robustness across material types and generalization beyond idealized conditions.
%
%

We quantitatively evaluate our method on scenes with known reflectance (e.g. color charts) and known shading (e.g. object imaged under identical lighting but painted differently). We further test on visible-thermal pairs simulated from the  MIT Intrinsic dataset~\cite{MITIntrinsic}. Finally, we demonstrate qualitative results on complex indoor and outdoor scenes with notable improvements over both physics-based and learning-based methods trained on auxiliary datasets. Our dataset of visible-thermal image pairs can also provide supervision for learning methods in real-world scenes. 

\textit{Limitations}: 
Our method assumes dominant diffuse reflection, heat arising primarily from light absorption, and the absence of multiple colored illuminations. 
Performance can also be affected by the limited SNR and resolution of inexpensive microbolometer thermal cameras, particularly under weak illumination or in dynamic scenes. 
Despite these limitations, as thermal cameras improve and become more ubiquitous, understanding the interplay between light and heat holds strong potential for vision and graphics.

\section{Related Work}
\label{sec:rw}
\subsection{Thermal Imaging for Physics-Based Vision}
Thermal cameras have recently emerged as a powerful complement to visible sensing across geometry~\cite{narayanan2024shape,thr_photometric_stereo,shape_radiation_2022_CVPR}, materials~\cite{thr4material_2023_CVPR}, and appearance~\cite{bao_heat-assisted_2023}. 

JoLHT-Video~\cite{joint_light_and_heat_transport} showed that transient thermal video provides analytical cues for estimating absorbed and reflected light. However, it relies on (i) active lighting control, infeasible in-the-wild, (ii) visible-only illumination that excludes common  sources (sun and incandescent bulbs), and (iii) radiometric calibration between visible and thermal cameras.
In contrast, our method uses a \emph{single} steady-state thermal image to extract reliable ordinal constraints without requiring video, calibration or lighting control.

\subsection{Intrinsic Image Decomposition (IID)}

\textbf{Early optimization-based approaches:} Retinex-style~\cite{land1977retinex} methods rely on stringent assumptions that hinder generalization—such as smooth shading or reflectance~\cite{barrow1978recovering}, chromaticity-preserving shading variations~\cite{gevers2003reflectance,finlayson2004color,das2022intrinsic}, or local intensity similarity implying shared reflectance~\cite{IID_optimize_cvpr2011}.


\noindent\textbf{Learning-based approaches:}
Unsupervised learning based methods that decorrelate albedo and shading ~\cite{liu2020unsupervised} or that enforce albedo consistency across changing illumination ~\cite{li2018learning} improve upon hand-crafted priors. 
Supervised learning-based models are primarily trained on synthetic datasets~\cite{li2018cgi,Li_2021_openrooms,Roberts_2021_hyperism,gta}, which provide ground-truth albedo and shading but face a significant sim-to-real gap. 
Existing real-world datasets~\cite{bell14IIW,kovacs17SAW,wu2023maw} offer sparse annotations used by models (e.g., to predict albedo ordinalities~\cite{learning_ordinal_iccv2015}), but are limited to small-scale indoor scenes. 
Intrinsic-v1~\cite{careagaIntrinsic-v1} expands to more diverse data by using model predictions as pseudo-ground truth, albeit imperfectly.
Recent works~\cite{kocsis2023intrinsic,zeng2024rgbx,intrinsicDiffusion} leverage diffusion priors for IID, yet as noted in~\cite{careagaIntrinsic-v2}, they suffer from hallucination.

\noindent\textbf{Using auxiliary sensors:}
\citet{cheng2019non} used near-infrared (NIR) images as shading proxies, but NIR albedo often varies across materials (albeit less than visible) and modern efficient lighting such as LEDs hardly emits NIR, limiting generality and applicability.
\citet{sato2023unsupervised} used intensity of sparse LiDAR returns and enforce consistency with estimated albedos, yet LIDAR operates in NIR where albedo differs from visible~\cite{cheng2019non}.
While such NIR cues help in specific cases, our approach exploits the complementary relation between visible (reflected light) and thermal (proxy for absorbed light), enabling broader applicability.
 

\begin{figure}[ttt]
  \captionsetup{position=top, skip=3pt}
  \centering  
  \setlength{\tabcolsep}{1pt} 
  \begin{tabularx}{\linewidth}{@{}YYYYYY@{}}
    Visible & Thermal & Edges & Points & Shading & Albedo \\
    \includegraphics[width=\linewidth,height=1.4\linewidth]{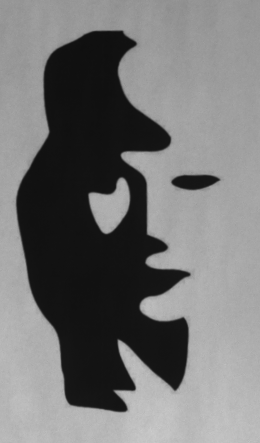} &
    \includegraphics[width=\linewidth,height=1.4\linewidth]{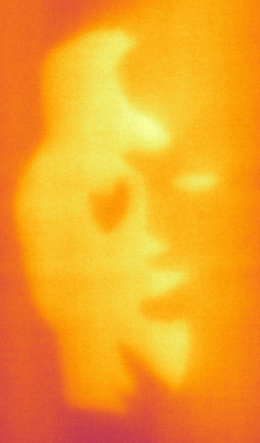} &
    \includegraphics[width=\linewidth,height=1.4\linewidth]{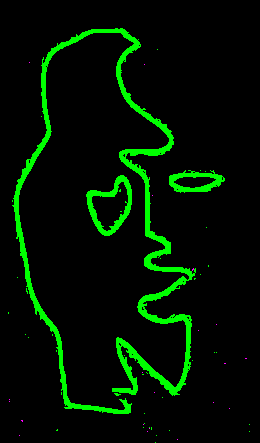} &
    \includegraphics[width=\linewidth,height=1.4\linewidth]{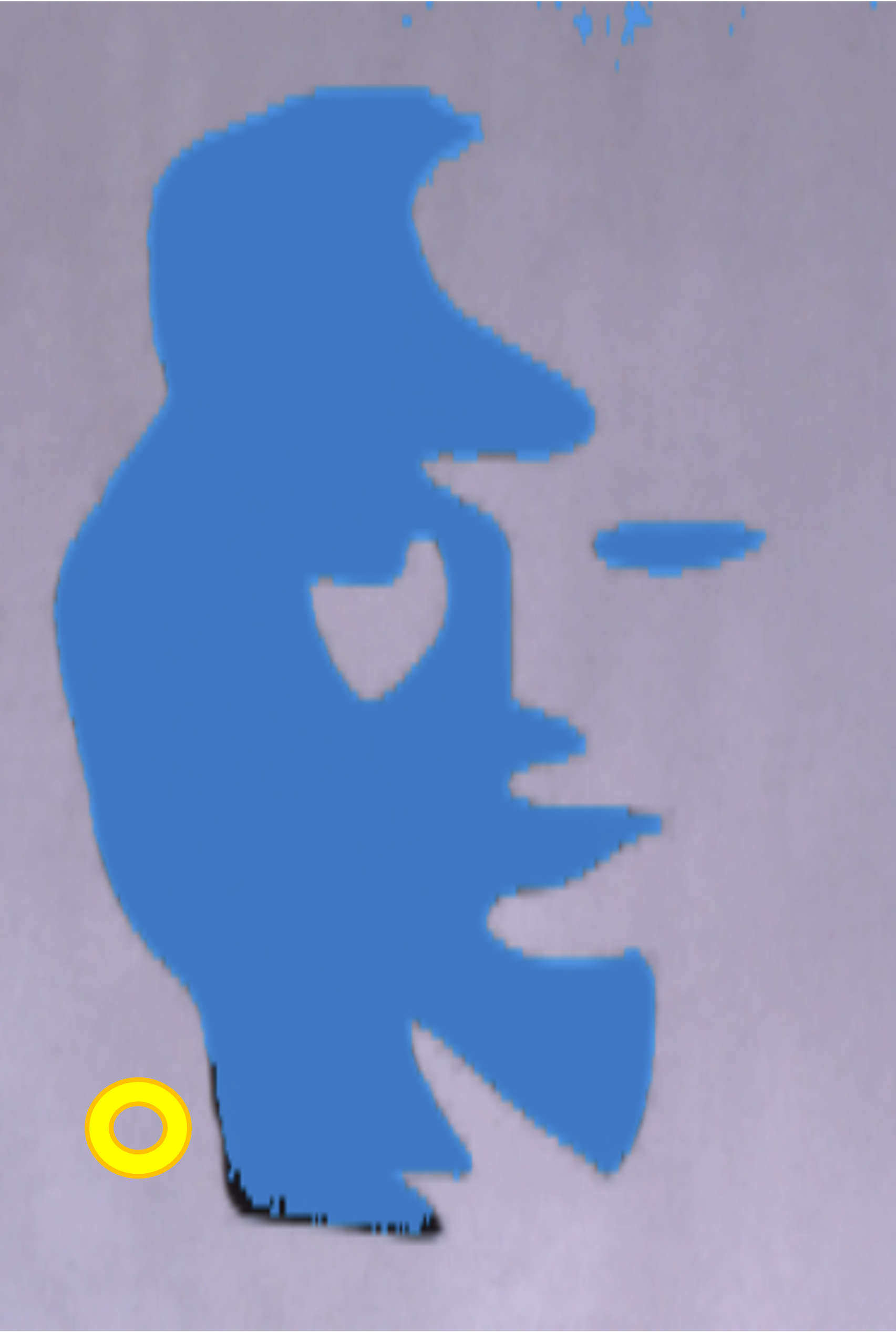} &
    \includegraphics[width=\linewidth,height=1.4\linewidth]{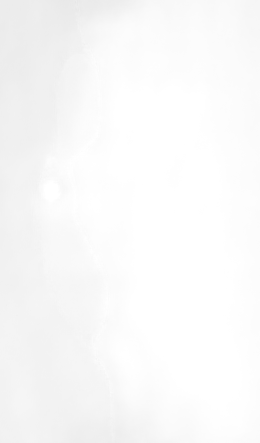} &
    \includegraphics[width=\linewidth,height=1.4\linewidth]{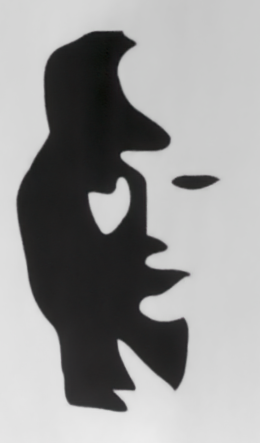} \\
    \includegraphics[width=\linewidth,height=1.4\linewidth]{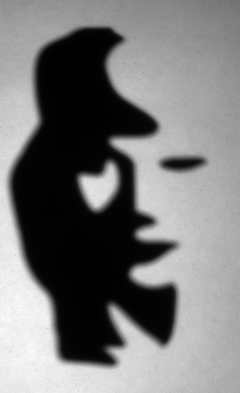} &
    \includegraphics[width=\linewidth,height=1.4\linewidth]{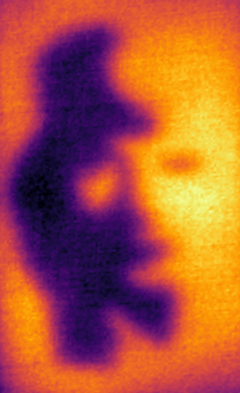} &
    \includegraphics[width=\linewidth,height=1.4\linewidth]{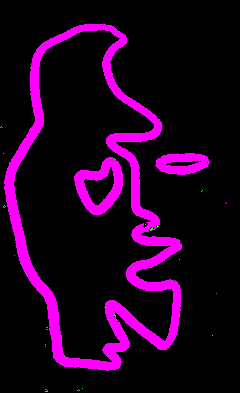} &
    \includegraphics[width=\linewidth,height=1.4\linewidth]{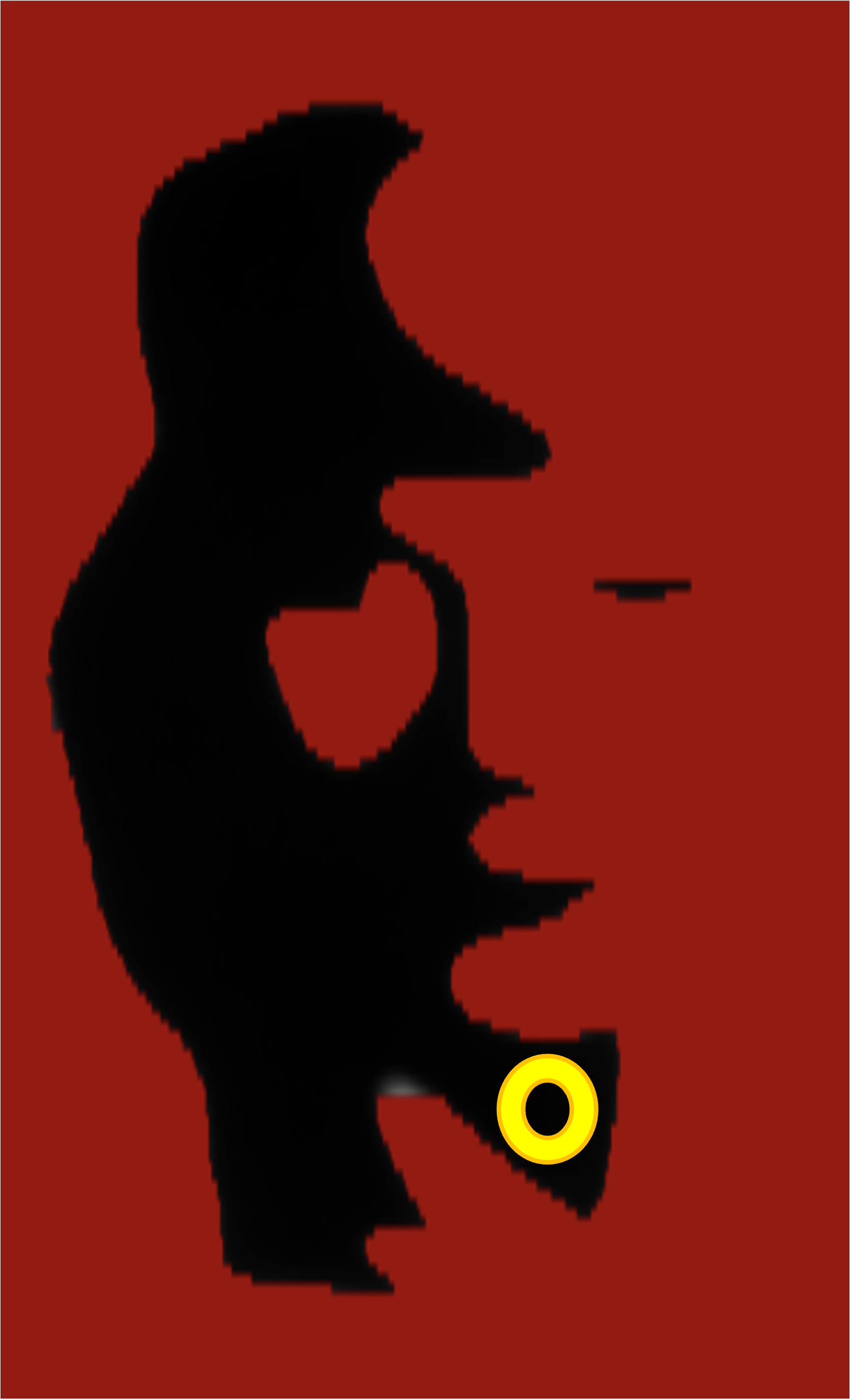} &
    \includegraphics[width=\linewidth,height=1.4\linewidth]{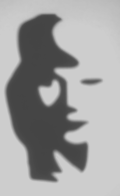} &
    \includegraphics[width=\linewidth,height=1.4\linewidth]{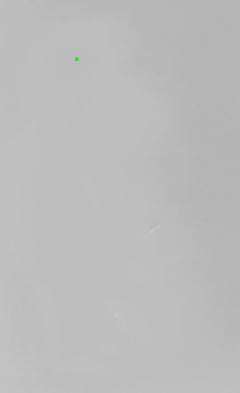} \\
  \end{tabularx}
  \caption{
    Printed (top) vs. projected (bottom) Roger Shepard’s illusion~\cite{MindSights}. 
    \textbf{Top}: a \textit{printed} paper lit by an incandescent bulb, where reflectance variations reveal a saxophone player. 
    \textbf{Bottom}: the same pattern \textit{projected} onto a uniform cardboard, where modulated shading reveals a lady’s face. 
    This comparison highlights the albedo-shading ambiguity and motivates modeling light–heat transport: reflectance induces inverse visible–thermal ordinalities, while shading yields consistent ones. 
    Columns 3-4 show classified \textcolor{customGreen}{albedo-} / \textcolor{customMagenta}{shading-}dominant edges (\autoref{sec:local_constraint}) and points of \textcolor{customBlue}{lower albedo} / \textcolor{customRed}{higher shading} than \protect\ycircle ~(\autoref{sec:non-local_constraint}).
    Our method decomposes correctly (right), whereas baselines fail (see supplementary).
    }   
  \label{fig:teaser}
  \vspace{-0.2in}
\end{figure}

\section{Theory of Visible-Thermal Ordinality}
\label{sec:theory}


We present the theoretical relationship between a visible and thermal image pair and show that the ordinality of their pixel intensities directly convey the ordinality of the underlying albedo or shading, as illustrated in ~\autoref{fig:teaser}.
We first consider visible illumination (e.g., an LED), then extend our analysis to include invisible illumination component (e.g., infrared light in incandescent bulbs and sunlight).

\subsection{Visible-only Illumination}
\label{subsec:vis_illum}
Consider an opaque Lambertian scene imaged by visible and thermal cameras. 
The visible intensity at a pixel $x$ is:
\begin{equation}
    I_v(x) = g \rho(x) \eta(x),
    \label{eq:vis_irr}
\end{equation}
where $\rho$ is the albedo (reflectance), $\eta$ the shading (irradiance), and $g = k/\pi$ a global scale determined by the camera gain $k$.
For brevity, we omit $x$ denoting a single pixel. 

Light not reflected is absorbed by the surface and converted into heat, producing a heat source of intensity:
\begin{equation}
    \mathcal{H} = (1 - \rho)\eta.
    \label{eq:heat_gen}
\end{equation}

\vspace{-0.05in}
This heat propagates 
through conduction, convection and radiation according to the heat transport equation.
Though $\mathcal{H}$ is not directly measurable, it can be inferred from surface temperature, which is indirectly observed by a thermal camera.
 JoLHT-Video~\cite{joint_light_and_heat_transport} modeled the light-heat transport using thermal video of the heating process to estimate $\mathcal{H}$ and solve for $\rho$ and $\eta$.
In contrast, we use a single thermal image $I_t$ near thermal equilibrium, easily attainable within seconds under stable lighting.

Together, $I_v$ and the absorbed heat image $\mathcal{H}$ impose local and non-local constraints on albedo and shading.

\subsubsection{Local (Edge) Constraints}
\label{sec:local_constraint}
The spatial gradient of the visible image can be written as:
\begin{equation}
    \nabla I_v = g(\nabla \rho) \eta + g\rho (\nabla \eta).
    \label{eq:vis_edge}
\end{equation}
For most edges in natural images, one of the two terms on the right dominates—edges arise primarily from either albedo or shading variations~\cite{land1977retinex,gevers2003reflectance}.
This creates a fundamental ambiguity, but the spatial gradients of the heat source provide complementary information:
\begin{equation}
    \nabla \mathcal{H} = (-\nabla \rho) \eta + (1-\rho) \nabla \eta.
    \label{eq:heat_edge}
\end{equation}
From ~\eqref{eq:vis_edge} and ~\eqref{eq:heat_edge}, we have
\begin{subequations}
\begin{align}
    &\textbf{Albedo Edge} (\nabla \eta \rightarrow 0) \textbf{:} \:  \textrm{sign}(\nabla I_v) = -\textrm{sign}(\nabla \mathcal{H}), \\
    &\textbf{Shading Edge} (\nabla \rho \rightarrow 0) \textbf{:} \:   \textrm{sign}(\nabla I_v) = \textrm{sign}(\nabla \mathcal{H}).
\label{eq:edge_classify}    
\end{align}
\end{subequations}
This yields a simple criterion to distinguish albedo- and shading-dominant edges using the heat intensity image.

\subsubsection{Non-Local (Point-Pair) Constraints}
\label{sec:non-local_constraint}
We generalize the above gradient analysis to compare point pairs, i.e., any two distinct pixels $x_i$ and $x_j$ in the scene.
\vspace{-0.15in}
\begin{subequations}
\begin{align}
    I_v(x_i) = g\rho(x_i) \eta(x_i), &\; \mathcal{H}(x_i) = (1 - \rho(x_i)) \eta(x_i), \\
    I_v(x_j) = g\rho(x_j) \eta(x_j), &\; \mathcal{H}(x_j) = (1 - \rho(x_j)) \eta(x_j).
\end{align}   
\label{eq:point_pair}
\end{subequations}
When a pixel’s visible intensity is lower (or higher) than another’s while its thermal intensity is higher (or lower), the pixel’s albedo is correspondingly lower (or higher).
\begin{proposition}[Albedo Ordinality]
For pixels $x_i, x_j$ with visible and heat intensities as in ~\eqref{eq:point_pair}, if $I_v(x_i) < I_v(x_j)$ and $\mathcal{H}(x_i) > \mathcal{H}(x_j)$, then $\rho(x_i) < \rho(x_j)$, and vice versa. 
\label{thm:albedo_ordinal}
\end{proposition}
\vspace{-0.08in}

Conversely, when both visible and heat intensities are lower (or higher), its shading is also lower (or higher).

\vspace{-0.05in}
\begin{proposition}[Shading Ordinality]
For pixels $x_i, x_j$ with visible and thermal intensities as in~\eqref{eq:point_pair}, if 
$I_v(x_i) < I_v(x_j)$ and $\mathcal{H}(x_i) < \mathcal{H}(x_j)$, then $\eta(x_i) < \eta(x_j)$, and vice versa.
\label{thm:shading_ordinal}
\end{proposition}
\vspace{-0.08in}

Proofs for \autoref{thm:albedo_ordinal} and \autoref{thm:shading_ordinal} are in the supplementary.
The ordinalities here rely on $\mathcal{H}$, the heat from absorbed visible light. Next, we incorporate invisible light and relate $\mathcal{H}$ to thermal image intensity, $I_t$.

\subsection{Visible and Invisible Illumination}
\label{subsec:vis_invis_illum}
\vspace{-0.05in}
Common light sources such as sunlight and incandescent lamps emit significant invisible radiation (e.g., infrared).  
While the visible camera captures only reflected light within its spectral response, heat generation arises from absorbed energy across all wavelengths.  
Thus, the equation for the heat source intensity has an additional term as follows:
\begin{equation}
    \mathcal{H} = (1 - \rho_v)\eta + (1-\rho_i)\frac{l_i}{l_v} \eta,
    \label{eq:heat_irr}
\end{equation}
where $\rho_i$ is the average albedo in the invisible band, $l_i/l_v$ is the ratio of light intensity in the invisible and visible spectra.

While albedo variations are prominent in the visible spectrum, their counterparts in the infrared are much smaller~\cite{ChoeG}. Thus, we assume that $\rho_i$ is locally constant within a region, allowing \eqref{eq:heat_irr} to be approximated as:
\begin{equation}
    \mathcal{H} = (\beta - \rho_v) \eta, \quad \textrm{s.t.} \quad \beta = 1 + (1 - \rho_i) l_i/l_v.
\end{equation}
As $\beta$ is locally constant, ~\eqref{eq:edge_classify} still holds as $\nabla \mathcal{H}$ is invariant to a constant offset in $\mathcal{H}$.
Also, as $\beta > 1$, ~\autoref{thm:albedo_ordinal} and ~\autoref{thm:shading_ordinal} holds whenever $\beta$ is same for the two points. 




\subsection{Relating heat intensity to a single thermal image}
\label{subsec:S_to_It}
\vspace{-0.05in}
While heat intensity is not directly observable, we show that a thermal image is a reliable proxy for ordinality constraints.

The heat transport equation at a surface point is:
\begin{equation}
    \mathcal{C}_h \frac{\partial T}{\partial t} = \mathcal{H} + h_c (T_a - T) + 4\epsilon \sigma T_s^3 (T_s - T) + \kappa \Delta T,
    \label{eq:hte}
\end{equation}
where $\mathcal{C}_h$ is the heat capacity, $T$ the surface temperature, $t$ time, $h_c$ the convection coefficient, $T_a$ the air temperature, $\epsilon$ the surface emissivity, $\sigma$ the Stefan-Boltzmann constant, $T_s$ the surrounding temperature, $\kappa$  the thermal conductivity, and $\Delta$ denotes the Laplacian operator along the surface. 
%
A static scene under constant lighting reaches thermal equilibrium when the left side of ~\eqref{eq:hte} is zero, giving
\begin{equation}
    \mathcal{H} = (h_c + 4\epsilon \sigma T_s^3)T - \kappa \Delta T - (h_c T_a + 4\epsilon \sigma T_s^4).
    \label{eq:heat_intensity}
\end{equation}

The image intensity measurement $T_t$ made by a thermal camera is related to the temperature $T$ as follows:
\begin{equation}
    I_t = \epsilon U(T) + (1-\epsilon) U(T_s),
    \label{eq:thr_img}
\end{equation}
where $U$ denotes the thermal camera's response function.
Linearizing $U$ as $U(T) = p_1 T + p_2$ in ~\eqref{eq:thr_img}, we get
\begin{equation}
    T =  a_1 I_t - a_2 \quad \text{s.t.} \; a_1 = \frac{1}{\epsilon p_1},\; a_2 = \frac{p_2 + p_1T_s(1-\epsilon)}{\epsilon p_1}.
    \label{eq:linear_T}
\end{equation}
Substituting ~\eqref{eq:linear_T} in ~\eqref{eq:heat_intensity}, we get
\begin{equation}
    \mathcal{H} = c_1 I_t - c_2 \Delta I_t - c_3,
\end{equation}
where $c_1 = \frac{h_c + 4\epsilon \sigma T_s^3}{\epsilon p_1}$, $c_2 = \frac{\kappa}{\epsilon p_1}$, and $c_3 = (h_c + 4\epsilon \sigma T_s^3)(\frac{p_2+p_1T_s(1-\epsilon)}{\epsilon p_1}) + (h_cT_a+4\epsilon \sigma T_s^4)$. 

The thermal properties such as $\epsilon$, and $\kappa$ have small variations irrespective of the variation in albedo~\cite{doi:https://doi.org/10.1002/9783527693306}.
The environmental variables such as $h_c, T_a$, and $T_s$ are also similar.
Therefore, $c_1, c_2$ and $c_3$ are similar within a region. 
Also, thermal conductivity of many common materials, excluding metals, is low. 
Likewise, the Laplacian of a temperature field at steady state has a much smaller magnitude than absolute temperatures~\cite{doi:https://doi.org/10.1002/9783527693306}. 
Therefore, we ignore the conduction term. Then, 
as $c_1 > 0$, the ordinal relationships between $\mathcal{H}$ at two points is the same as that of $I_t$.
\begin{proposition}
In local regions, $c_1, c_2$ and $c_3$ are constant so that for any two pixels $x_i, x_j$,  if $\mathcal{H}(x_i)$ is less (or more) than $\mathcal{H}(x_j)$, then $I_t(x_i)$ is also less (or more) than $I_t(x_j)$.
\label{assum:S_to_It_preserved}
\end{proposition}

\subsection{Ordinality of Albedo and Shading}
\label{subsec:alb_shad_ordinality}
\vspace{-0.05in}
Using Prop.~\ref{assum:S_to_It_preserved}, we can extend the results from ~\autoref{eq:edge_classify} to use thermal image intensities, as summarized below:
\begin{subequations}
\small
\begin{align}
    \textbf{Albedo Edge} (\nabla \eta = 0) \textbf{:} \: & \textrm{sign}(\nabla I_v) = -\textrm{sign}(\nabla I_t), \\
    \textbf{Shading Edge} (\nabla \rho = 0) \textbf{:} \: &  \textrm{sign}(\nabla I_v) = \textrm{sign}(\nabla I_t).
\end{align}
\label{eq:edge_classify_It}
\end{subequations}
Similarly, we extend \autoref{thm:albedo_ordinal} and \autoref{thm:shading_ordinal} to thermal image intensities, yielding the following ordinal relationships:
\begin{subequations}
\small
\begin{align}
    I_v(x_i) > I_v(x_j) , I_t(x_i) > I_t(x_j) & \Rightarrow \eta(x_i) > \eta(x_j), \\
    I_v(x_i) < I_v(x_j) , I_t(x_i) < I_t(x_j) & \Rightarrow \eta(x_i) < \eta(x_j), \\
    I_v(x_i) > I_v(x_j) , I_t(x_i) < I_t(x_j) & \Rightarrow \rho(x_i) > \rho(x_j), \\
    I_v(x_i) < I_v(x_j) , I_t(x_i) > I_t(x_j) & \Rightarrow \rho(x_i) < \rho(x_j).
\end{align}
\label{eq:alb_shad_ordinals}
\end{subequations}

\vspace{-0.25in}
\section{Method}
\label{sec:method}
\vspace{-0.05in}


Using the ordinalities as loss functions, we optimize the albedo and shading from a visible-thermal image pair.
Let $I_v$ be a $k-$channel visible image and $I_t$ be the corresponding aligned thermal image. 
Let $\hat{\rho}$ and $\hat{\eta}$ be an estimate of the $k-$channel albedo and grayscale shading. 
Let $\bar{I}_v$ and $\bar{\rho}$ be the grayscale image and albedo estimate, respectively.






\subsection{Local (Edge) Loss}
\label{subsec:edge_loss}
\vspace{-0.03in}
Using \autoref{eq:edge_classify_It}, we label edges (A for albedo, S for shading) based on their local visible–thermal gradients (\autoref{fig:teaser}):
\begin{equation}
    \mathcal{C}(x) = \begin{cases}
        \text{A} & |\nabla \bar{I}_v| > \epsilon_m, |\frac{\nabla \bar{I}_v \nabla I_t}{\| \nabla \bar{I}_v \| \| \nabla I_t \|} | > \epsilon_p, \\
        \text{S} & |\nabla \bar{I}_v| > \epsilon_m, |\frac{\nabla \bar{I}_v \nabla I_t}{\| \nabla \bar{I}_v \| \| \nabla I_t \|} | < \epsilon_p, \\
    \end{cases}
    \label{eq:edge_classifier}
\end{equation}
where $\epsilon_m$ suppresses textureless regions and $\epsilon_p$ thresholds the cosine similarity between visible and thermal gradients. 

Before computing $\nabla I_t$, we apply Gaussian smoothing to reduce noise while maintaining gradient consistency.
With the class labels above, we formulate an edge loss that penalizes albedo gradients at shading-dominant pixels and vice versa, where $\Omega$ denotes all image pixels:
\vspace{-4pt}
\begin{equation}
    \small
    \mathcal{L}_{\text{edge}} (\bar{\rho}, \hat{\eta}, \mathcal{C}) = \frac{1}{|\Omega|} \Big[ \sum_{\mathcal{C}(x) = S}
         \| \nabla \bar{\rho}(x)\|^2 
         + \sum_{\mathcal{C}(x) = A}
         \| \nabla \hat{\eta}(x)\|^2 \Big].
\end{equation}


\subsection{Non-Local (Point-Pair) Loss}
\label{subsec:point_pair_loss}

During optimization, we use Poisson disk sampling~\cite{bell14IIW} to generate random point pairs across the image. 
%
%
%
%
%
%
%
Using \autoref{eq:alb_shad_ordinals}, each pair $(x_i, x_j)$ is assigned a class label based on their normalized intensity differences $\delta I_v$ and $\delta I_t$:
\vspace{-6pt}
\begin{equation}
\mathcal{P}(x_i,x_j)=
\begin{cases}
S_+ & \!\!\delta I_v>\epsilon_d,~\delta I_t>\epsilon_d,\\
S_- & \!\!\delta I_v<-\epsilon_d,~\delta I_t<-\epsilon_d,\\
A_+ & \!\!\delta I_v>\epsilon_d,~\delta I_t<-\epsilon_d,\\
A_- & \!\!\delta I_v<-\epsilon_d,~\delta I_t>\epsilon_d,\\
\end{cases}
\label{eq:pair-classes}
\end{equation}
\vspace{-2pt}
where $\delta I_x=\tfrac{I_x(x_i)-I_x(x_j)}{Z_x}$ with normalization $Z_x$ so that threshold $\epsilon_d$ is relative.
The ordinal loss is a hinge-based formulation that enforces separation beyond a margin $\varepsilon_m$:
\vspace{-4pt}
\begin{equation}
\small
\mathcal{L}_{\text{ord}}
=\tfrac{1}{|\mathcal{P}|}\!\!\sum_{(x_i,x_j)}\!
\begin{cases}
\max(\hat{\eta}_j-\hat{\eta}_i+\varepsilon_m,0),&\mathcal{P}(x_i,x_j)\!=\!S_+,\\
\max(\hat{\eta}_i-\hat{\eta}_j+\varepsilon_m,0),&\mathcal{P}(x_i,x_j)\!=\!S_-,\\
\max(\bar{\rho}_j-\bar{\rho}_i+\varepsilon_m,0),&\mathcal{P}(x_i,x_j)\!=\!A_+,\\
\max(\bar{\rho}_i-\bar{\rho}_j+\varepsilon_m,0),&\mathcal{P}(x_i,x_j)\!=\!A_-.
\end{cases}
\label{eq:ord-loss}
\end{equation}

\vspace{-0.1in}
\subsection{Regularization using Deep Image Prior}
\label{subsec:dip}


In complex real scenes, 
thermal noise can corrupt subtle gradients, 
and ordinal constraints alone cannot fully determine absolute albedo or shading values—they only restrict the solution space.
Therefore, we adopt a variant of the Deep Image Prior~\cite{ulyanov20dip} to parameterize albedo and shading, leveraging the inherent architectural prior in a randomly initialized network for regularization.

We employ a Double-DIP (DDIP) architecture ~\cite{gandelsman2018doubledip} with two networks $\mathcal{N}(z_A, \Theta_A), \mathcal{N}(z_S, \Theta_S)$ to parameterize albedo and shading, respectively. 
Each uses a convolutional encoder–decoder with skip connections~\cite{ulyanov20dip}. 
$\Theta_A, \Theta_S$ are randomly initialized model weights and $z_A,z_S$ are randomly sampled input noise vectors.
The albedo network outputs a $3$-channel image bounded to $[0,1]^3$ via a sigmoid activation, while the shading network predicts a single channel constrained by a non-negativity penalty.
We freeze $z_A$ and $z_S$ while only optimizing for $\Theta_A$ and $\Theta_S$.

\subsection{Optimization}

Our complete objective function is as follows.
\begin{multline}
\mathcal{L}(\hat{\rho}, \hat{\eta}, I_v, I_t) = \|\hat{\rho} \cdot \hat{\eta} - I_v \|_2 + \lambda_1  \mathcal{L}_{\text{edge}}(\bar{\rho}, \hat{\eta}, \mathcal{C}(\bar{I}_v, I_t)) + \\ \lambda_2 
\mathcal{L}_{\text{ord}}(\bar{\rho}, \hat{\eta}, \mathcal{P}(\bar{I}_v, I_t)),
\end{multline}
where $\lambda_1, \lambda_2 > 0$ are the respective loss weights.
The thermal image is used only for edge or point pair losses, which operate on the mean albedo. The reconstruction loss is defined on the $3-$channel image.

\begin{figure}[ht]
    \centering
    \includegraphics[width=1\linewidth]{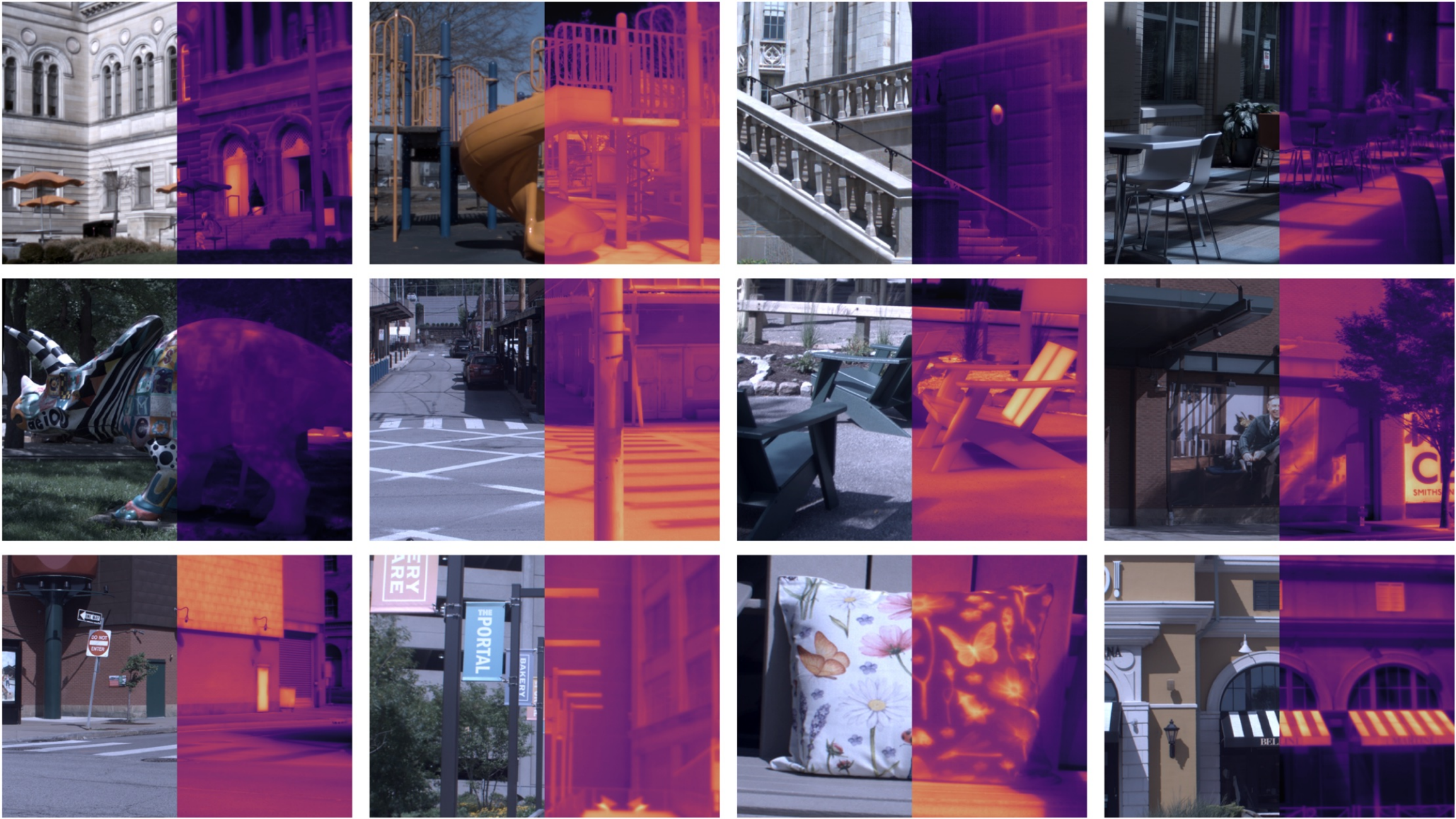}
    \vspace{-0.25in}
    \caption{Visible–thermal image pair examples in the \dataset{} dataset, covering diverse scenes including parks, schools, cathedrals, plazas, museums, and various urban streets.
    }
    \label{fig:dataset_mosaic}
    \vspace{-0.20in}
\end{figure}

\begin{figure*}[!ht]
    \centering
    \includegraphics[width=1\linewidth]{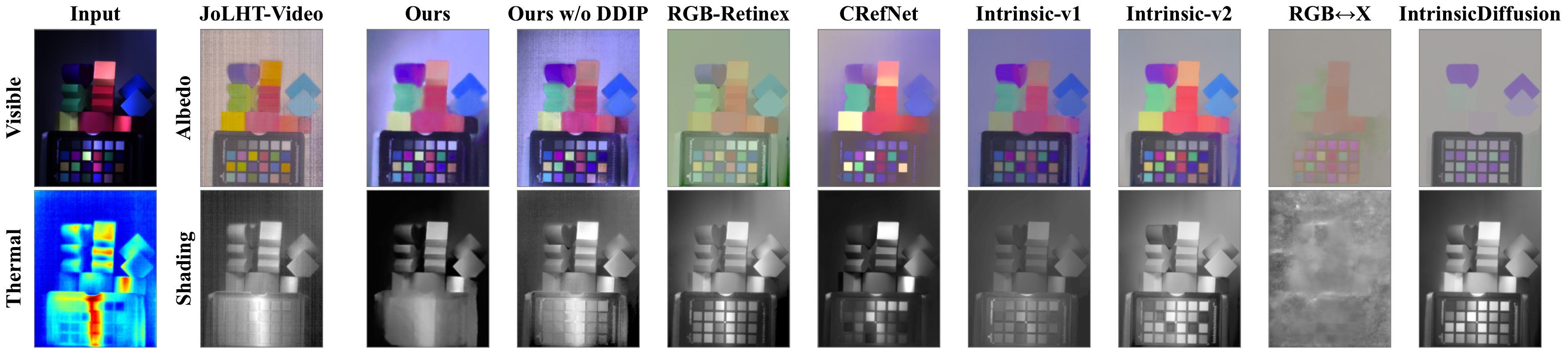}
    \vspace{-20pt}
    \caption{Results on a color-chart scene in JoLHT-Video dataset. 
    Our method recovers the smooth line-light shading across the color chart.
    }
    \label{fig:wooden_block_comparison}
     \vspace{-0.1in}
\end{figure*}

\begin{figure*}[t]
    \centering
    \includegraphics[width=1\linewidth]{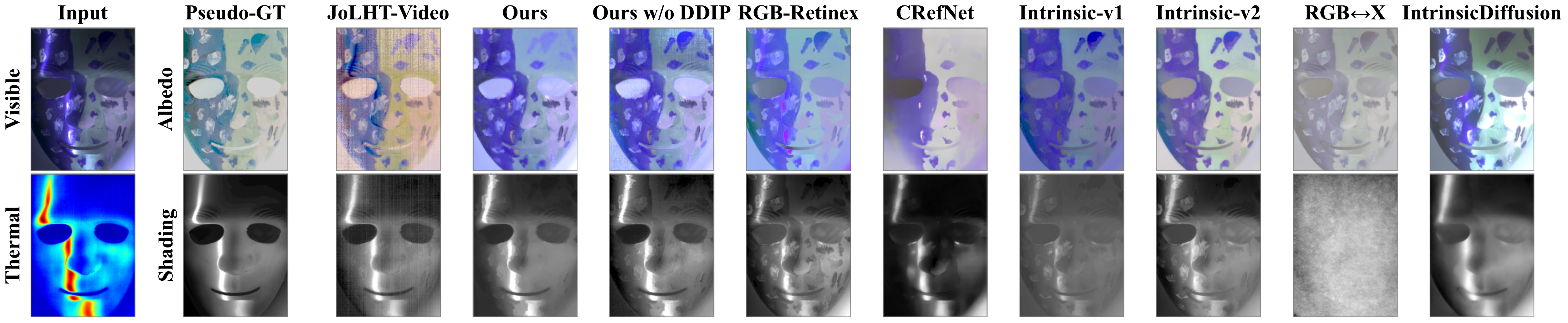}
    \vspace{-20pt}
    \caption{Results on \emph{Painted-Mask} scene in JoLHT-Video dataset.
    Baselines show albedo texture in shading or highlight artifacts in albedo. 
    }
    \label{fig:pseudo-gt_mask}
    \vspace{-0.2in}
\end{figure*}

\begin{figure*}[p!]
    \centering
    \captionsetup{skip=2pt}
    \includegraphics[width=\linewidth
    ]{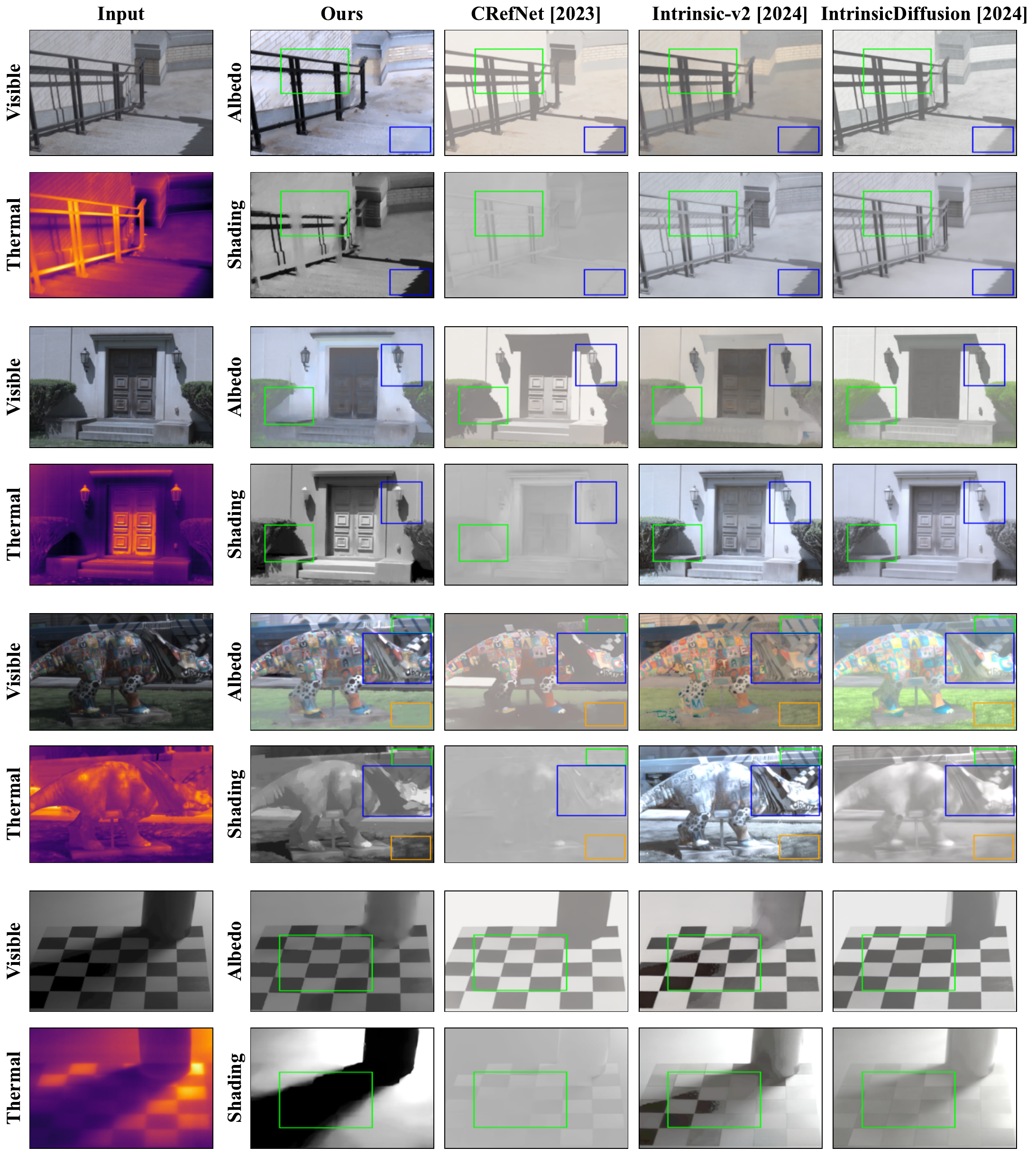}
    \caption{
    \textbf{Qualitative comparisons to state-of-the-art baselines.} 
    The first two scenes show how our method removes cast shadows from albedo (e.g., shadow of handrail in case 1, lanterns in case 2). The next three demonstrate our ability to eliminate albedo texture from shading (e.g., rhino statue texture in case 3, checkerboard pattern in case 4). 
    But baselines struggle with these challenges, despite their advantage of pre-training on large datasets, whereas our approach relies solely on physics-based information in a single thermal image.
    Baselines often over-smooth albedo and shading (e.g., smooth albedo on detailed ground and walls, flat shading on grass) due to reliance on priors. 
    Diffusion-based baselines can offer appealing visual quality but sacrifice faithfulness (e.g. hallucinated albedo texture on the rhino statue in case 3).
    Images are tonemapped for visualization. 
    Key differences are highlighted in bounding boxes.
    \emph{More examples and baselines appear in the supplementary.}
    }
    \label{fig:qualitative_comparison}
\end{figure*}

\begin{table*}[t]
\centering
\small
\renewcommand{\arraystretch}{1.2}
\setlength{\tabcolsep}{10pt}
\caption{Results of si-MSE ($\downarrow$) reported at $10^{-2}$ across datasets (\autoref{subsec:quant_eval}).
\textbf{Best} and \underline{second} highlighted. Our method surpasses all learning-based approaches despite using no learned priors and achieves performance comparable to JoLHT-Video, which demands transient thermal video under controlled illumination. 
N/A indicates unavailable data, and \ding{55} denotes non-applicability.}
\vspace{-0.14in}

\begin{tabular}{@{}lccc|ccc@{}}
\toprule
\multirow[t]{3}{*}{
  \begin{minipage}[t]{0.28\linewidth}\raggedright
  \textbf{Method}\\
  \footnotesize{
      \quad \optmark \; Optimization-based \\[-1pt]
      \quad \learnmark \; Learning-based \\[-1pt]
      \quad \physmark \; Physics-based (w/ auxiliary sensor)
  }
  \end{minipage}
}
& \multicolumn{3}{c}{\textbf{JoLHT-Video Dataset}}
& \multicolumn{3}{c}{\textbf{Color Chart w/ Different Illumination}} \\[-2pt]
\cmidrule(lr){2-4}\cmidrule(lr){5-7}

& \multicolumn{2}{c}{Painted Mask} & Color Charts
& White LED & Incandescent & Sunlight \\[-4pt]

& Albedo & Shading & Albedo
& Albedo & Albedo & Albedo \\
\hline

\optmark \; RGB-Retinex~\scriptsize{\cite{color-retinex}~(TPAMI'06)} &
25 & 0.30 & 3.4 &
2.42 & 2.33 & 2.73 \\

\optmark \; Opt-LocalSmooth~\scriptsize{\cite{IID_optimize_cvpr2011}~(CVPR'11)} &
45 & 0.35 & 7.1 &
2.41 & 4.21 & 2.04 \\

\learnmark \; IntrinsicDiffusion~\scriptsize{\cite{intrinsicDiffusion}~(SIGGRAPH’24)} &
37 & 0.25 & 2.9 &
4.12 & 3.33 & 4.85 \\

\learnmark \; RGB$\leftrightarrow$X~\scriptsize{\cite{zeng2024rgbx}~(SIGGRAPH’24)} &
30 & 0.37 & 2.8 &
4.07 & 5.31 & 4.59 \\

\learnmark \; Intrinsic-v2~\scriptsize{\cite{careagaIntrinsic-v2}~(ToG’24)} &
27 & 0.17 & 2.8 &
\underline{1.25} & 4.36 & 4.17 \\

\learnmark \; Intrinsic-v1~\scriptsize{\cite{careagaIntrinsic-v1}~(ToG’23)} &
30 & 0.21 & 3.8 &
1.55 & 2.72 & 4.97 \\

\learnmark \; CRefNet~\scriptsize{\cite{crefnet}~(TVCG’23)} &
38 & 0.23 & 8.8 &
1.79 & \underline{2.29} & \underline{1.98} \\

\physmark \; NIR-Priors~\scriptsize{\cite{cheng2019non}~(ICCV’19)} &
N/A & N/A & N/A &
\ding{55} & 2.46 & 2.08 \\

\physmark \; JoLHT-Video~\scriptsize{\cite{joint_light_and_heat_transport}~(CVPR’24)} &
\textbf{8.4} & \textbf{0.05} & \textbf{2.0} &
N/A & \ding{55} & \ding{55} \\

\rowcolor{gray!20}
\physmark \; Ours &
\underline{11} & \underline{0.10} & \underline{2.7} &
\textbf{0.37} & \textbf{1.06} & \textbf{1.19} \\

\bottomrule
\end{tabular}

\vspace{-0.18in}
\label{tab:simulationComparison}
\label{tab:our_color_chart}
\label{tab:quantitative_merged}
\label{tab:quantitative_all}
\end{table*}

\section{\dataset Dataset}
\label{sec:dataset}
\vspace{-0.04in}
Existing IID datasets lack thermal modalities, while current visible–thermal datasets (e.g., captured from vehicles or drones) focus on dynamic objects (people, cars) or non-light-absorption based heat sources (engines, people) that are out of scope for this work. So, we collected 600 visible–thermal image pairs (\autoref{fig:dataset_mosaic}) across diverse stationary scenes under varying illumination to validate our method.

\textbf{Imaging System.}\,
We co-locate a FLIR Boson thermal camera ($512\times640$ resolution, $24^\circ$ HFOV, $\leq 50\text{mK}$ NEDT) with an IDS UI-3130 color camera ($600\times 800$ resolution, $27^\circ$ HFOV) using a gold dichroic mirror (BSP-DI-25-2). For distant outdoor scenes, the cameras are placed side by side and aligned via homography.

\textbf{Data Acquisition and Preprocessing.}\,
We captured 20 exposure-bracketed color images with geometrically spaced exposure times and merged them into a linear HDR image~\cite{debevec1997} after edge-aware demosaicing in OpenCV. 
Five frames were averaged to suppress sensor noise. The visible HDR and thermal images were aligned via homography.


Our dataset contributes in two key aspects:
(i) a large collection of high-quality real-world outdoor images with diverse albedo–shading combinations (vs. predominantly indoor/synthetic prior datasets);
(ii) abundant pseudo-ground-truth albedo and shading ordinalities for training and evaluation.
As shown in \autoref{sec:expert_validation}, the thermal image produces reliable albedo/shading ordinalities across \emph{arbitrary pixel pairs}—previously only available in limited form due to costly human labeling 
(IIW~\cite{bell14IIW}).

\vspace{-0.05in}
\section{Experiments}
\label{sec:results}
\vspace{-0.04in}

\textbf{Datasets.}
As IID datasets lack associated thermal images, we construct the \dataset dataset for qualitative evaluation.
Obtaining ground truth albedo and shading for real-world scenes is impractical. 
Therefore, for quantitative evaluation, we collected images of a color chart under different illuminations: white LED light, incandescent bulb and sunlight.
We also evaluate on the JoLHT-Video dataset~\cite{joint_light_and_heat_transport},
which contains four color-chart scenes under varied illuminations and a \emph{Painted-Mask} scene.
In supplementary, we evaluate on MIT-Intrinsic~\cite{MITIntrinsic} dataset by simulating an ideal thermal image using pseudo-ground truth.




\textbf{Metrics.} We use the scale-invariant Mean Square Error (si-MSE) to evaluate albedo and shading quantitatively. 


\textbf{Baselines.} We compare with state-of-the-art methods in three categories.
\emph{Learning-based:} Diffusion-based IntrinsicDiffusion~\cite{intrinsicDiffusion} and RGB$\leftrightarrow$X~\cite{zeng2024rgbx}, CNN-based Intrinsic-v1~\cite{careagaIntrinsic-v1} and Intrinsic-v2~\cite{careagaIntrinsic-v2}, and Transformer-based CRefNet~\cite{crefnet}.
\emph{Physics-based:} NIR-Priors~\cite{cheng2019non}, requiring a paired NIR image, and JoLHT-Video~\cite{joint_light_and_heat_transport}, demanding transient thermal video under controlled illumination.
\emph{Optimization-based:} RGB-Retinex~\cite{color-retinex} and Opt-LocalSmooth~\cite{IID_optimize_cvpr2011}.
IntrinsicDiffusion, RGB$\leftrightarrow$X, and Intrinsic-v2 output colorful shading, while others grayscale.

\vspace{-0.06in}
\subsection{Qualitative Evaluation}
\vspace{-0.04in}

\autoref{fig:qualitative_comparison} presents comparisons with state-of-the-art baselines across various scenes.
The first two cases demonstrate our ability to remove cast shadows from albedo (e.g., handrail and lantern shadows), while the next two highlight disentangling albedo texture from shading (e.g., rhino statue texture and checkerboard pattern).
The final example is an homage to the classic Adelson’s Checker-Shadow Illusion~\cite{adelson1}, where our method successfully separates the shadowed checker region from the cylinder shading.

Learning-based baselines often over-smooth albedo and shading due to strong statistical priors, producing flat grass shading or overly uniform wall colors. In contrast, our physics-based approach, guided solely by a single thermal image, better preserves details such as block-wise albedo variation, concrete texture, and natural shading gradients.

\autoref{fig:wooden_block_comparison} and \autoref{fig:pseudo-gt_mask} show results on JoLHT-Video dataset. Our method recovers smooth line-light shading comparable to JoLHT-Video using only a single thermal image, while other baselines show clear albedo or shading leakage.


\vspace{-0.02in}
\subsection{Quantitative Evaluation}
\label{subsec:quant_eval}
\vspace{-0.02in}
\subsubsection{Validating Ordinalities}
\label{sec:expert_validation}
\vspace{-0.05in}
We validated our ordinality theory on a wide range of real-world materials and scenes via expert annotations, and confirm statistically with a spectral reflectance dataset~\cite{kokaly2017usgs} that invisible component rarely overturn these ordinalities.

\noindent\textbf{Patch Ordinality on Various Materials.}
\label{sec:material_validation}
The evaluation included 20 patches from CUReT dataset~\cite{curet_dataset} and common objects (painted aluminum, plastic, wood, silk, leather, cloth, plaster, etc.). These patches were placed in different orientations under artificial and natural lighting. Experts confidently labeled 865 ordinalities across patches. Our prediction matched the expert labels with 98.59\% accuracy in sunlight (albedo: 99.37\%, shading: 97.01\%) and 96.82\% under white-LED (albedo: 94.62\%, shading: 100\%).

\noindent\textbf{Point-Pair Ordinality on Diverse Scenes.}
We further evaluated on 100 real-world scenes in \dataset dataset (\autoref{sec:dataset}), spanning materials such as stone, concrete, grass, vegetation, painted metal, plastic, and wood. 
Experts labeled the ordinalities in albedo or shading of 20 randomly sampled point pairs per image using the visible image as reference.
Pairs with small intensity differences were excluded to avoid ambiguity. 
Experts confidently labeled 1,063 pairs and found 937 unclear.
Ignoring the latter, our theory achieved 98.95\% overall accuracy (albedo: 96.96\%, shading: 99.62\%), confirming the reliability of thermal-guided ordinal cues. More details are in the supplementary.

\vspace{0.02in}
\noindent\textbf{Statistical Robustness to Invisible Spectral Components.} Violations of ordinalities due to invisible reflectance {\it are possible but statistically unlikely}. Above we validated ordinalities with expert labels for many materials and lighting conditions with non-visible components. As a further evaluation, we used the Solar Spectral Irradiance (ASTM-G173) and 427 materials from USGS Spectral Reflectance dataset~\cite{kokaly2017usgs}, and confirmed that the ordinality of absorbed visible light matches the ordinality of total absorbed light (including UV and IR) in 94.2\% cases out of   all 90,951 or ${427 \choose 2}$ 
material pairs. The remaining 5.8\% of violations occur mostly when absorptances are nearly identical. 

\vspace{-0.05in}
\subsubsection{Validating IID Method}
\vspace{-0.05in}
We present evaluations on the color charts under different illuminations and JoLHT-Video dataset~\cite{joint_light_and_heat_transport}, and include in the supplementary (i) a simulation experiment to validate ordinality informativeness and (ii) an ablation study. 

\noindent\textbf{Color Chart under Different Illuminations.}
\label{sec:colorChart-lighting}
We imaged a color chart under white LED, incandescent and sunlight.
\autoref{tab:our_color_chart} shows our method outperforming baselines under all illuminations. 
Incandescent and sunlight experiments demonstrate our robustness to albedo variations even in the invisible band that influence absorbed light.
In contrast, physics-based baselines have limited applicability: JoLHT-Video assumes no invisible lighting component, and NIR-Priors requires NIR emission absent in white LEDs.

\noindent\textbf{Using JoLHT-Video Dataset.}
\label{sec:JoLHT-Video_data}
The dataset~\cite{joint_light_and_heat_transport} includes four color-chart scenes and a \emph{Painted-Mask} scene with pseudo ground-truth obtained following~\cite{MITIntrinsic}, which are considerably challenging due to the strong lighting variations from line light (\autoref{fig:wooden_block_comparison}, \autoref{fig:pseudo-gt_mask}). 
As shown in \autoref{tab:quantitative_merged}, our method outperforms all learning-based baselines without pre-trained priors, and achieves performance comparable to JoLHT-Video~\cite{joint_light_and_heat_transport}, which demands stricter conditions of calibrated transient thermal video and controlled lighting.


\begin{figure}[ttt]
 \captionsetup{position=top, skip=3pt}
 \centering
 
 \setlength{\tabcolsep}{2pt} 
 \begin{tabularx}{\linewidth}{@{}YYYY@{}}
   \includegraphics[width=\linewidth]{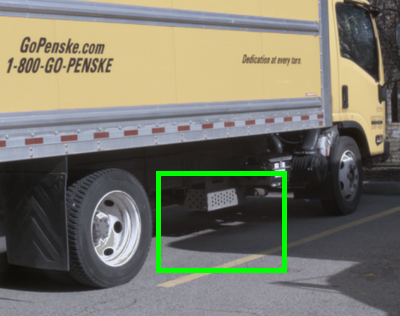} &
   \includegraphics[width=\linewidth]{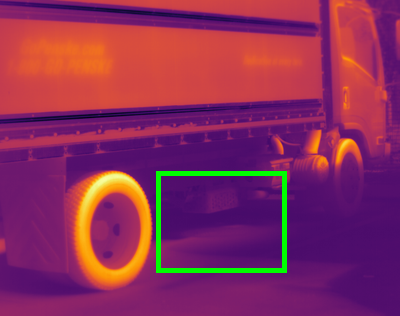} &
   \includegraphics[width=\linewidth]{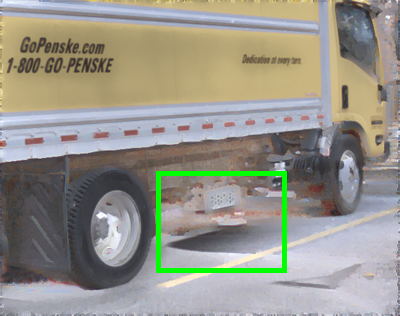} &
   \includegraphics[width=\linewidth]{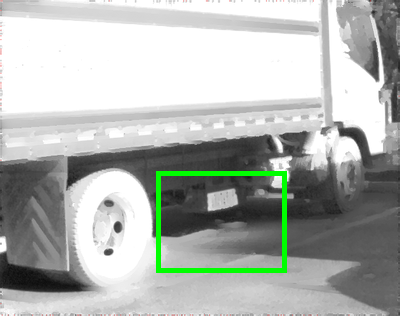} \\
   \includegraphics[width=\linewidth]{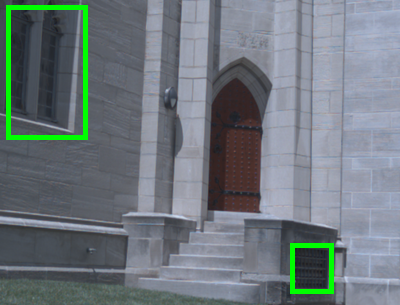} &
   \includegraphics[width=\linewidth]{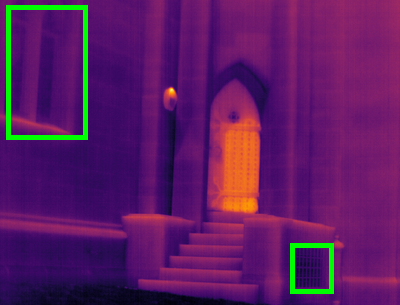} &
   \includegraphics[width=\linewidth]{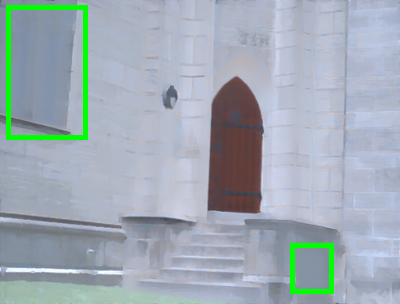} &
   \includegraphics[width=\linewidth]{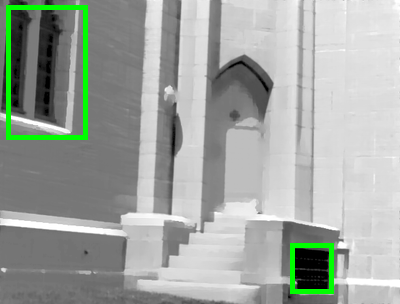} \\
   \includegraphics[width=\linewidth]{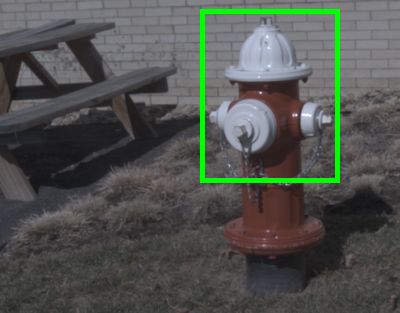} &
   \includegraphics[width=\linewidth]{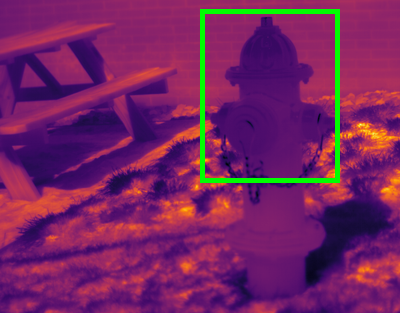} &
   \includegraphics[width=\linewidth]{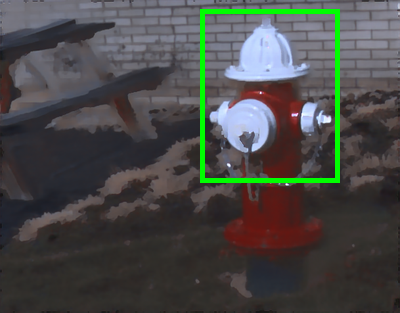} &
   \includegraphics[width=\linewidth]{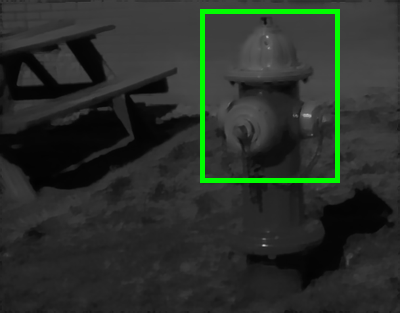} \\[-3pt]
   \small Visible & \small Thermal & \small Albedo & \small Shading 
 \end{tabularx}
 \vspace{-0.03in}
 \caption{ 
 Corner cases: 
 1) Heat from a truck engine, unrelated to light absorption, elevates the thermal intensity of the road beneath it. 
 2) Non-opaque cathedral windows violate the visible image formation model..
 3) The metallic fire hydrant has low emissivity, resulting in poor thermal SNR, and exhibits specular highlights that challenge the common Lambertian assumption in IID.
 }
 \label{fig:failure_cases}
 \vspace{-0.35in}
\end{figure}

\vspace{-0.08in}
\section{Limitations and Conclusion}
\vspace{-0.08in}
\label{sec:limitation_and_conclusion}
This work explores photometric cues encoded in a single auxiliary thermal image, and presents physics-based optimization for albedo-shading separation.
We showed its effectiveness on real scenes with a wide range of materials and lighting conditions. 
However, diffuse reflection dominates in these materials --- metals, transparent objects and mirrors violate the visible image formation model.
Our model also assumes that the heat arises primarily from light absorption --- heat generated otherwise internally (engines, humans) or externally (hot air blower or fire) is not modeled. 
It also does not handle multiple colored illuminations.
Finally, we rely on inexpensive microbolometer thermal cameras whose quality is lower compared to visible cameras --- low SNR due to insufficient heat generation (overcast skies, dynamic objects) can degrade performance. Failure cases are shown in \autoref{fig:failure_cases}, with additional analysis on low-light and non-equilibrium thermal conditions in the supplementary.
Despite these limitations, the improvements demonstrate the potential to scale supervision for learning algorithms. We hope our work inspires further exploration of light–heat interaction in computer vision and graphics. 

\section*{Acknowledgements}

This work was partly supported by NSF grants IIS210723, and NSF-NIFA AI Institute for Resilient Agriculture. We are sincerely grateful to Akihiko Oharazawa for his help with expert annotation, and to Sriram Narayanan and Gaurav Parmar for their insightful discussions.

{
    \small
    \bibliographystyle{ieeenat_fullname}
    \bibliography{main}
}

\clearpage
\setcounter{page}{1}
\setcounter{proposition}{0}
\maketitlesupplementary 
\newpage
\setcounter{section}{0}

\begin{table}[t]
\centering
\caption{Result of si-MSE ($\downarrow$) on simulated MIT-Intrinsic Dataset.}
\begin{tabular}{lcc}
\toprule
Method & Average Albedo & Average Shading \\
\midrule
Ours & $1.9\%$ & 0.5\% \\
\bottomrule
\end{tabular}
\label{tab:ours_mit_intrinsic}
\end{table}

\begin{figure}[ttt]
    \centering
    \begin{subfigure}{0.155\linewidth}
        \centering
        \includegraphics[width=\linewidth]{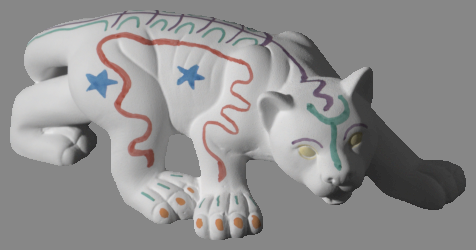}
        \includegraphics[width=\linewidth]{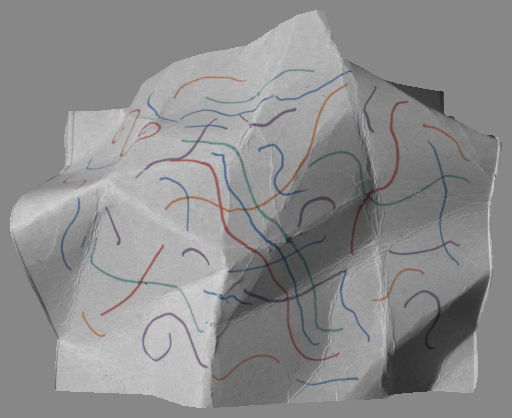}
        \caption*{\centering Input \\Visible}
    \end{subfigure}
    \hfill
    \begin{subfigure}{0.155\linewidth}
        \centering
        \includegraphics[width=\linewidth]{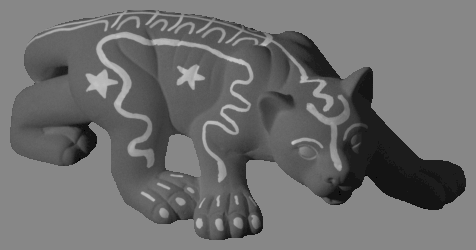}
        \includegraphics[width=\linewidth]{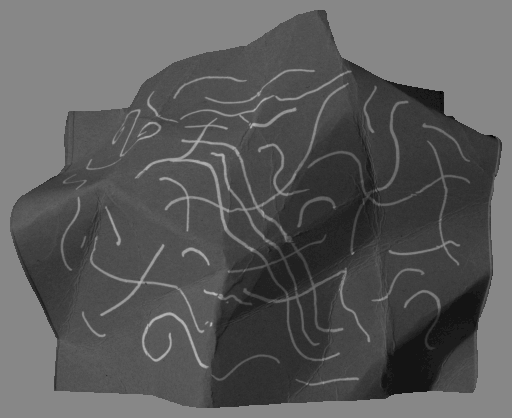}
        \caption*{\centering Simulated \\Thermal}
    \end{subfigure}
    \hfill
    \begin{subfigure}{0.155\linewidth}
        \centering
        \includegraphics[width=\linewidth]{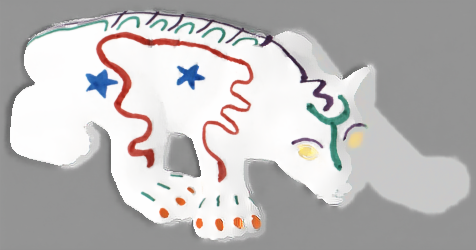}
        \includegraphics[width=\linewidth]{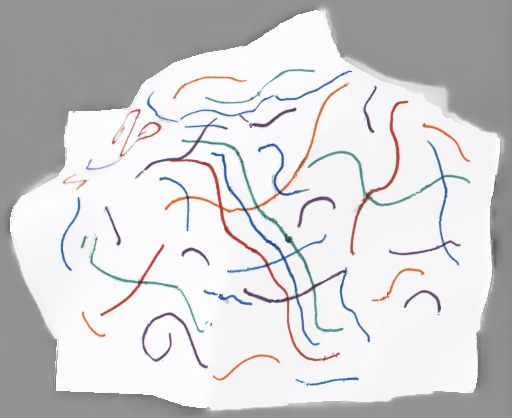}
        \caption*{\centering Our\\ Albedo}
    \end{subfigure}
    \hfill
    \begin{subfigure}{0.155\linewidth}
        \centering
        \includegraphics[width=\linewidth]{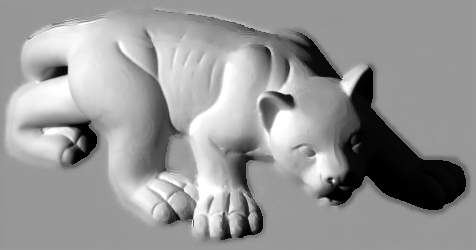}
        \includegraphics[width=\linewidth]{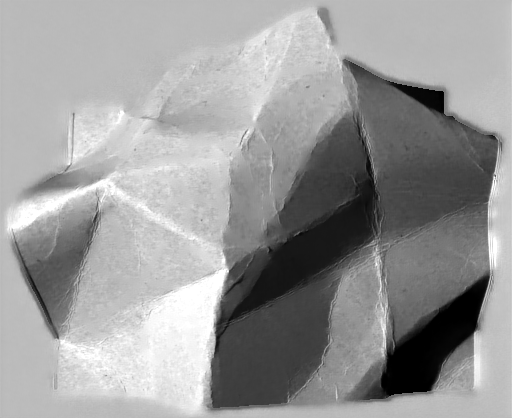}
        \caption*{\centering Our\\ Shading}
    \end{subfigure}
    \hfill
    \begin{subfigure}{0.155\linewidth}
        \centering
        \includegraphics[width=\linewidth]{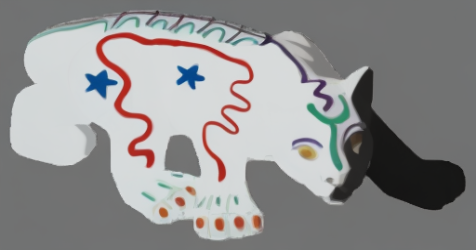}
        \includegraphics[width=\linewidth]{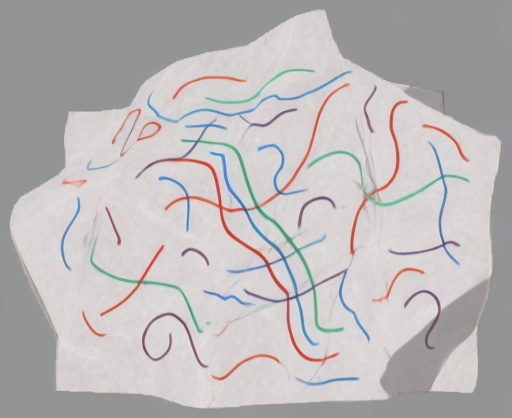}
        \caption*{\centering Baseline\\ Albedo}
    \end{subfigure}
    \hfill
    \begin{subfigure}{0.155\linewidth}
        \centering
        \includegraphics[width=\linewidth]{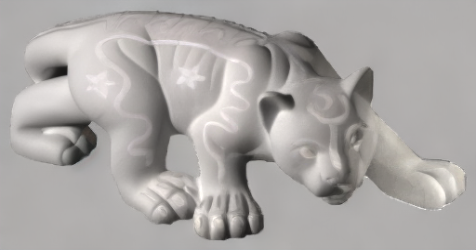}
        \includegraphics[width=\linewidth]{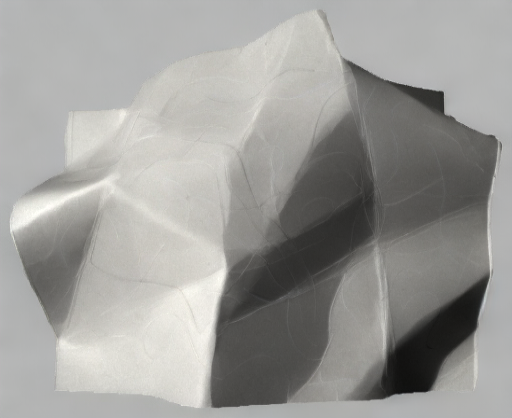}
        \caption*{\centering Baseline\\ Shading}
    \end{subfigure}
    \vspace{-5pt}
    \caption{Qualitative comparison with the best baseline on MIT-Intrinsic dataset~\cite{MITIntrinsic}, IntrinsicDiffusion~\cite{intrinsicDiffusion}
    . Our method excels without learning priors by leveraging the simulated thermal image.}
    \label{fig:qualitative_table_combined}
\end{figure}

\section{Simulated MIT-Intrinsic Dataset: Investigating the Informativeness of Ordinality}
\label{sec:mit-intrinsic}
The effectiveness of our method depends on two orthogonal factors: (1) the correctness of albedo–shading ordinalities derived from visible–thermal pairs, determined by imaging quality and the robustness of our theory to variations in physical properties, and (2) the informativeness of such ordinalities for intrinsic decomposition.
With (1) validated in \autoref{sec:expert_validation}, we isolate (2) by generating ideal absorbed-light images as thermal input using pseudo ground truth from the MIT-Intrinsics dataset~\cite{MITIntrinsic} via \autoref{eq:heat_gen}.
\autoref{tab:ours_mit_intrinsic} shows that a single ideal thermal image achieves a low average si-MSE on both albedo and shading.
\section{Ablation on Loss Terms and DDIP}
We ablated the loss functions and Double-DIP parameterization using pseudo ground truth from the \emph{Painted Mask} scene.
\autoref{tab:ablate_loss} shows
that combining ordinal, edge, and reconstruction losses with Double-DIP yields the best result.

\begin{table}[h]
\small
    \renewcommand{\arraystretch}{1.0}
    \centering
    \caption{Ablation study on loss terms and Double-DIP parameterization. We report the si-MSE for the \emph{Painted Mask} scene.}
    \vspace*{-0.1in}
    \label{tab:ablate_loss}
    \begin{tabular}{@{}cccccc@{}}
    \toprule
    \textbf{\(\mathcal{L}_{\text{recon}}\)} & \textbf{\(\mathcal{L}_{\text{edge}}\)} & \textbf{\(\mathcal{L}_{\text{ord}}\)} & \textbf{DDIP} & \textbf{Albedo $\downarrow$} & \textbf{Shading $\downarrow$} \\
    \midrule
    \ding{51} & \ding{51} & \ding{51} &\ding{51} &$\mathbf{1.1 \times 10^{-1}}$ & $\mathbf{9.7 \times 10^{-4}}$ \\
    \ding{51} & \ding{51} & \ding{51} &\ding{55} &$1.6 \times 10^{-1}$ & $32 \times 10^{-4}$ \\
    \ding{51} & \ding{51} & \ding{55}          &\ding{51} &$2.2 \times 10^{-1}$ & $18 \times 10^{-4}$ \\
    \ding{51} & \ding{55}          & \ding{51} &\ding{51} & $ 2.0\times 10^{-1}$ & $ 13\times 10^{-4}$ \\
    \ding{55}          & \ding{51} & \ding{51} &\ding{51} &$4.0 \times 10^{-1}$ & $79 \times 10^{-4}$ \\
    \ding{51} & \ding{55}          & \ding{55}          &\ding{51} &$3.3 \times 10^{-1}$ & $22 \times 10^{-4}$ \\
    \bottomrule
    \end{tabular}
\end{table}

\section{Proof for Propositions}
\begin{proposition}
Given two pixels with visible and heat intensities as in 
\text{Eq.\,6},
if $\mathcal{H}(x_i) > \mathcal{H}(x_j)$ and $I_v(x_i) < I_v(x_j)$, then $\rho(x_i) < \rho(x_j)$, and vice versa. 
\label{thm:albedo_ordinal}
\end{proposition}
\begin{proof}
Given
\begin{align}
    (1 - \rho(x_i)) \eta(x_i) >& (1-\rho(x_j))\eta(x_j) \\
    g\rho(x_i)\eta(x_i) <& g\rho(x_j) \eta(x_j).
\end{align}
Dividing the first eq. by the second and noting that all terms are positive, we get
\begin{align}
    \frac{1 - \rho(x_i)}{g\rho(x_i)} &> \frac{1 - \rho(x_j)}{g\rho(x_j)} 
\implies \rho(x_i) < \rho(x_j)
\end{align}
Proof for the complement is omitted for brevity.
\end{proof}

\begin{proposition}
Given two pixels with visible and heat intensities as in 
\text{Eq.\,6},
if 
$I_v(x_i) < I_v(x_j)$ and $\mathcal{H}(x_i) < \mathcal{H}(x_j)$, then $\eta(x_i) < \eta(x_j)$, and vice versa.
\label{thm:shading_ordinal}
\end{proposition}
\begin{proof}
Since multiplying an inequality by a positive scalar and adding two inequalities of same order preserves the order, we have
\begin{equation}
    \frac{I_v(x_i)}{g} + \mathcal{H}(x_i) < \frac{I_v(x_j)}{g} + \mathcal{H}(x_j).
    \label{eq:sum_S_I}
\end{equation}
From 
Eq.\,1 and Eq.\,2, 
note that $\frac{I_v(x)}{g} + \mathcal{H}(x) = \eta(x)$.
Substituting in 
Eq.\,24,
we can see that
\begin{equation}
    \eta(x_i) < \eta(x_j)
\end{equation}
Proof for the complement is omitted for brevity.
\end{proof}

\begin{figure}[H]
    \centering
    \begin{subfigure}{0.48\linewidth}
        \centering \includegraphics[width=\linewidth]{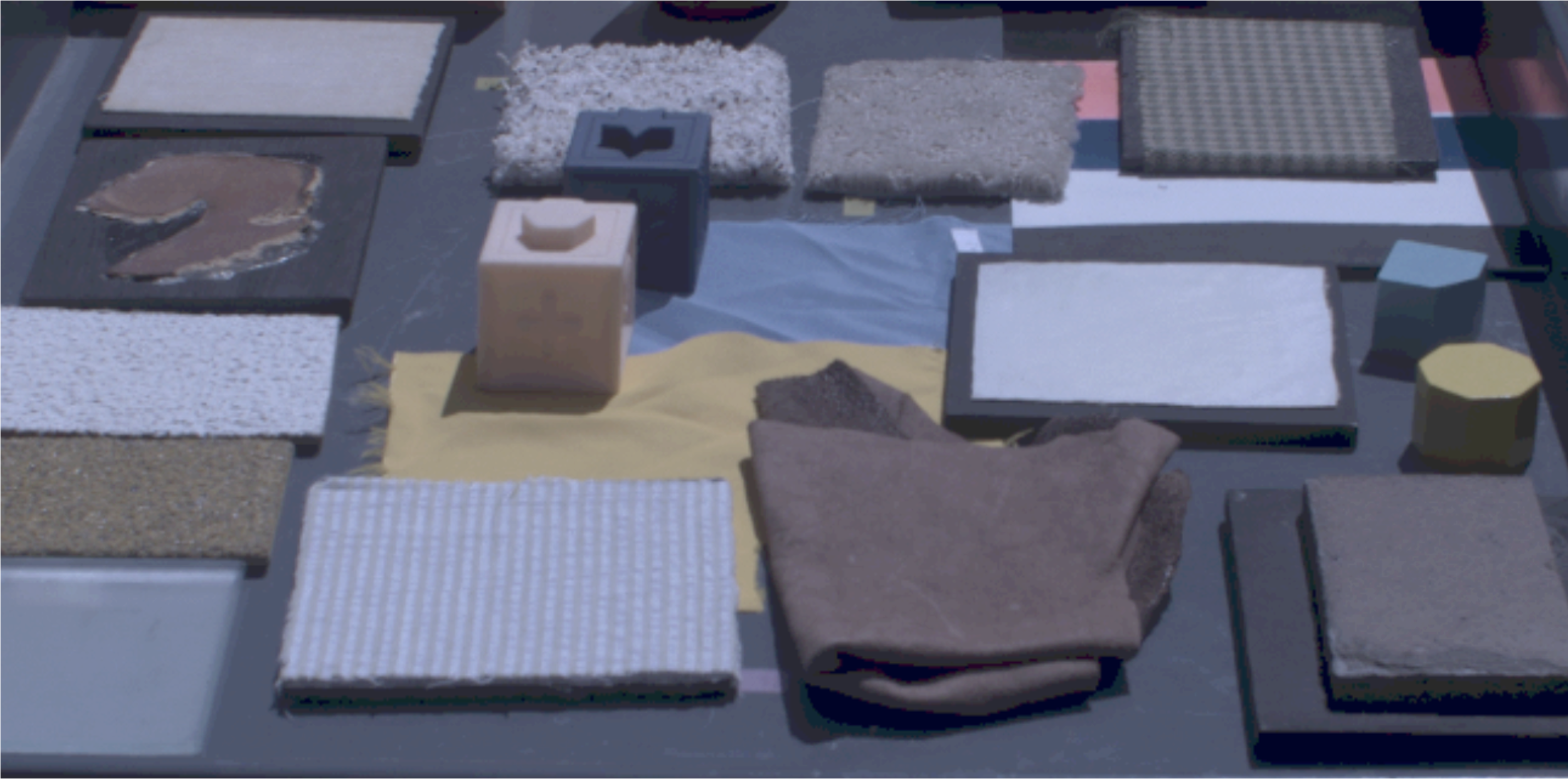}
        \caption*{Visible Image under Sunlight}
    \end{subfigure}
    \hfill
    \begin{subfigure}{0.48\linewidth}
        \centering \includegraphics[width=\linewidth]{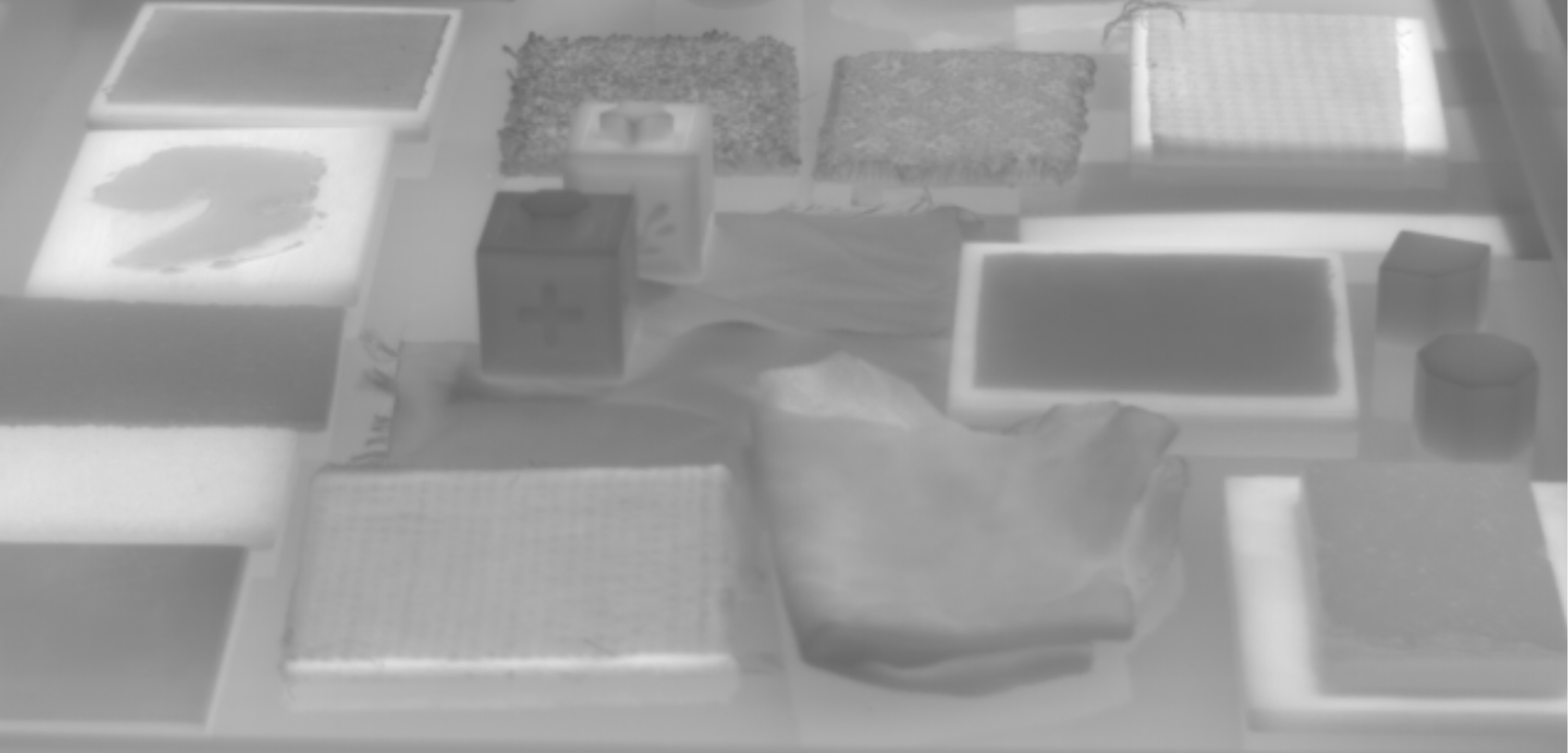}
        \caption*{Thermal Image under Sunlight}
    \end{subfigure}
    \\
    \begin{subfigure}{0.48\linewidth}
        \centering \includegraphics[width=\linewidth]{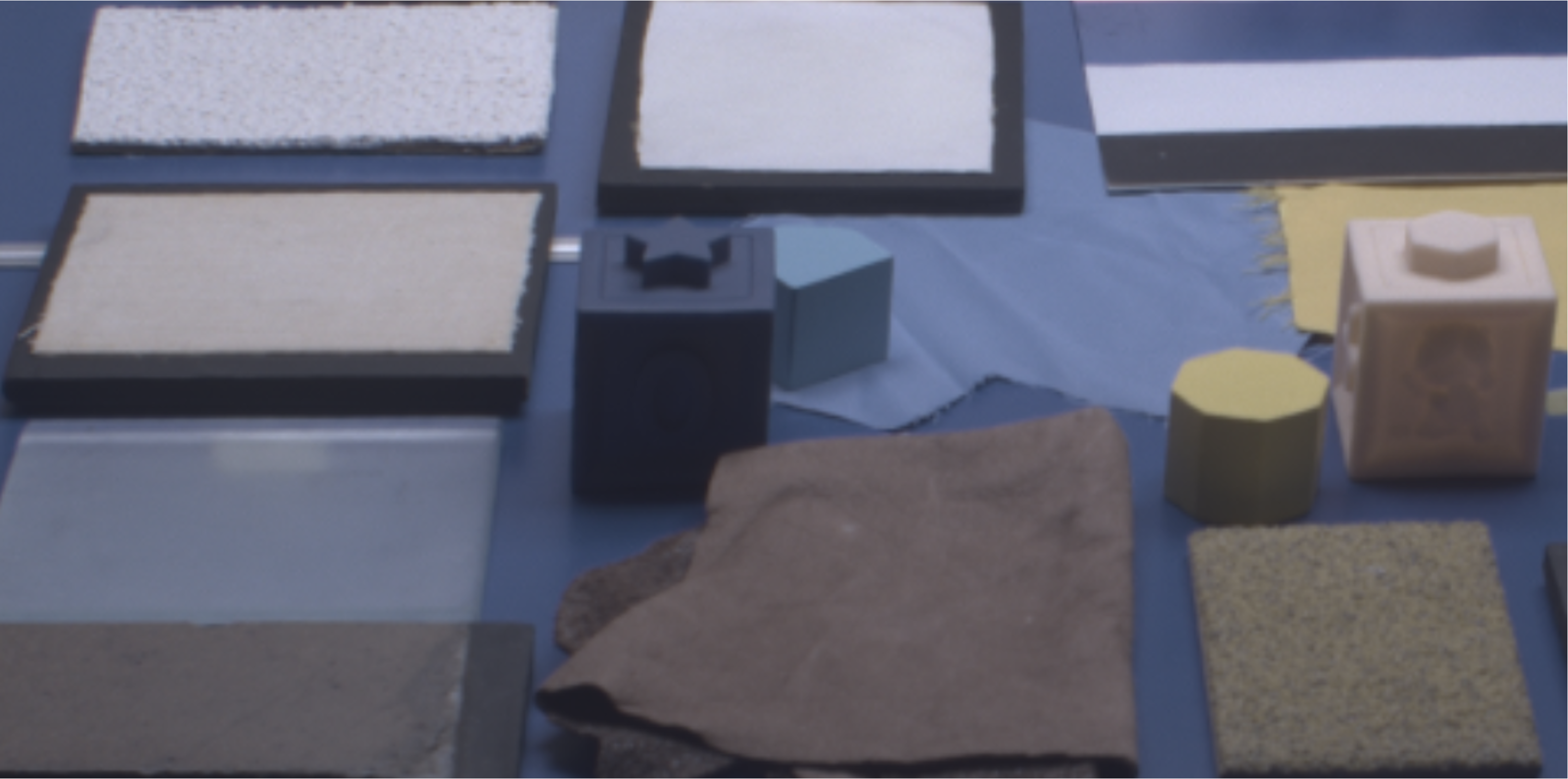}
        \caption*{Visible Image under White LED}
    \end{subfigure}
    \hfill
    \begin{subfigure}{0.48\linewidth}
        \centering \includegraphics[width=\linewidth]{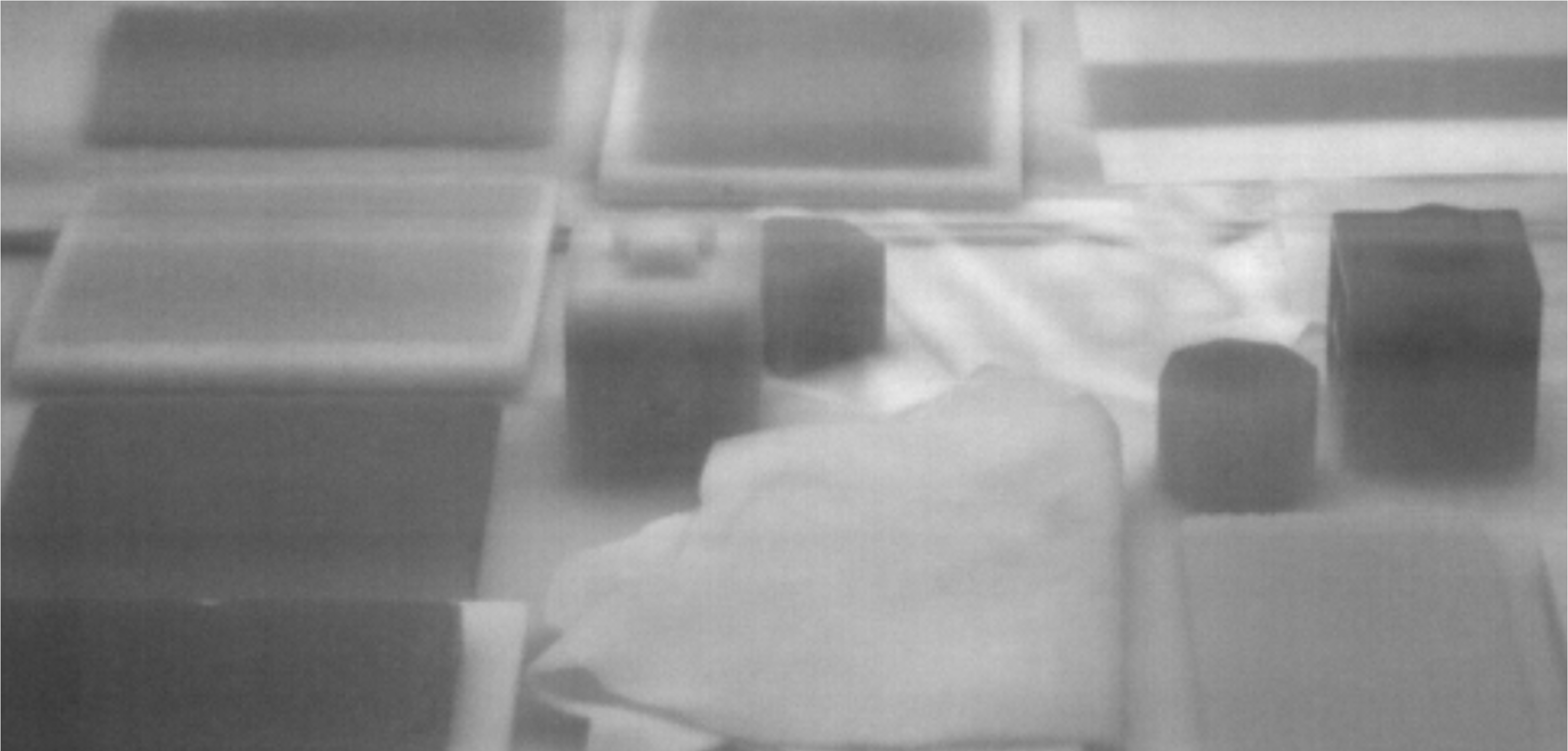}
        \caption*{Thermal Image under White LED}
    \end{subfigure}
    \vspace{-0.1in}
    \caption{An example scene in expert validation on ordinalities across diverse materials 
    (Sec.\,3.5).
    }
    \label{fig:material_validation}
    \vspace*{-0.15in}
\end{figure}

\begin{table}[ht]
\centering
\caption{The 20 materials used in ordinality validation.}
\label{tab:material_list}
\begin{tabular}{ll}
1. Terrycloth        & 11. Orange peel \\
2. Plaster           & 12. Wooden block \\
3. Felt              & 13. Yellow silk \\
4. Cork              & 14. Blue silk \\
5. Frosted glass     & 15. Painted aluminum can \\
6. Sponge            & 16. Painted metal handcart \\
7. Carpet            & 17. Plastic board w/ black paint \\
8. White leather     & 18. Plastic board w/ white paint \\
9. Brick             & 19. Beige rubber block \\
10. Suede leather    & 20. Blue rubber block \\
\end{tabular}
\end{table}

\section{Ordinality Validation on Diverse Materials}
To further examine how material properties affect the validity of our theory and assumptions, we created scenes with 20 common material samples from CUReT dataset~\cite{curet_dataset} and daily objects under sunlight and white-LED, and conducted expert validation on albedo-shading ordinalities. The detailed list of 20 materials used is shown in \autoref{tab:material_list}.

\section{Additional Limitation Analysis}
The key limitations of our method arise when the relationship between the absorbed heat from light ($S$) and the thermal image intensity ($I_t$) is violated, which can be summarized by three categories: external heat generation, non-opaque surfaces, and low signal-to-noise ratio (SNR).
\autoref{fig:failure_cases} shows representative failure cases.

\subsection{Low SNR in Low-Light Conditions}
To investigate how thermal image SNR is influenced by low-light condition, we captured visible-thermal image pairs of a color chart under an incandescent light at different distance. We measured illuminance using a light meter and computed si-MSE on the albedo decomposed by our method. As shown in \autoref{fig:si-MSE_vs_illuminance}, si-MSE decreases with increasing illuminance, indicating improved thermal SNR and decomposition quality under higher illumination. For reference, direct sunlight reaches about 100,000 lux, while overcast daylight is around 6,000 lux~\cite{Illuminance_overcast_sky}.

\vspace{-0.1in}
\begin{figure}[!ht]
    \centering
    \includegraphics[width=\linewidth]{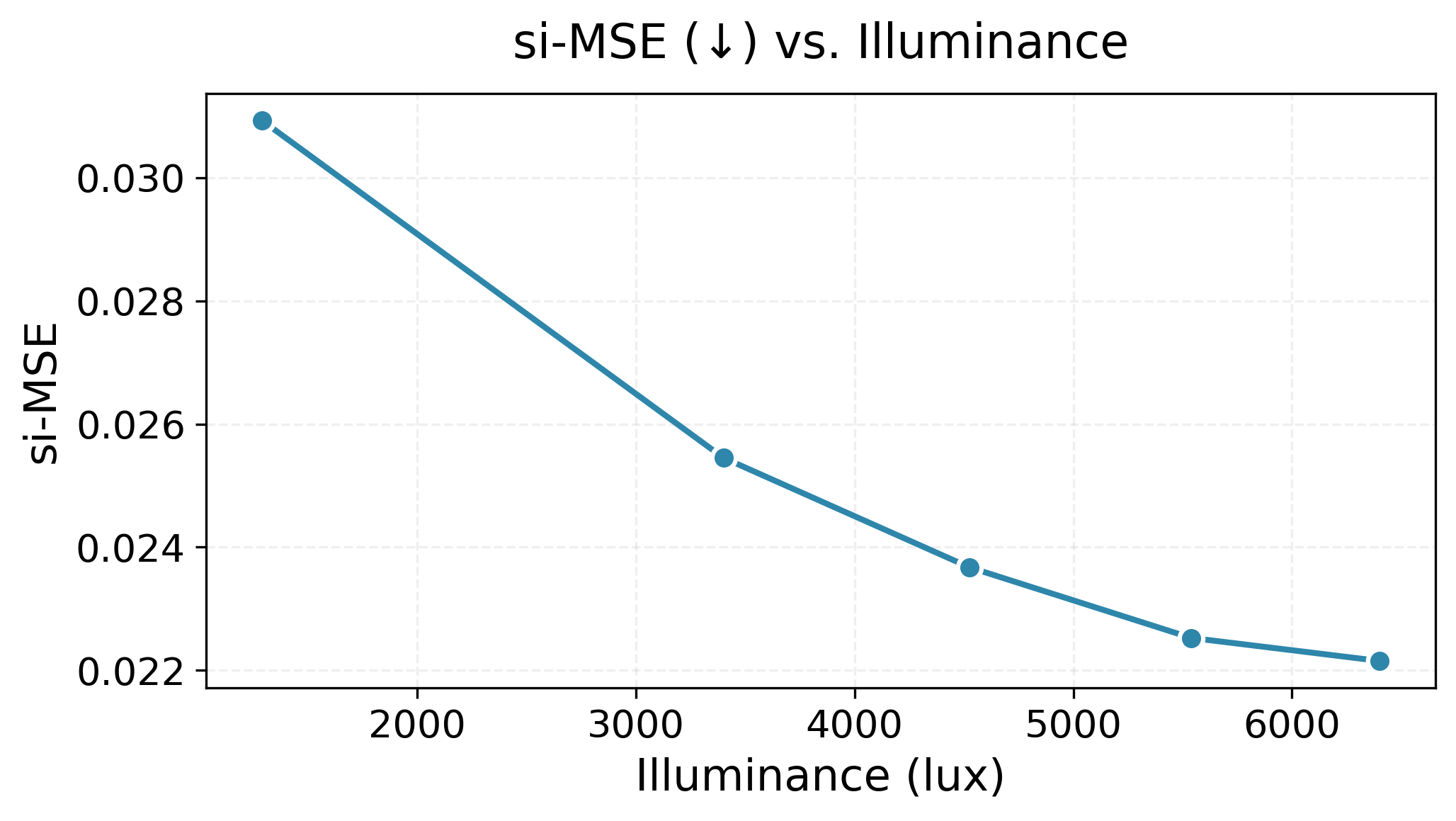}
    \vspace{-0.1in}
    \caption{Effect of illuminance on albedo decomposition accuracy. The albedo si-MSE ($\downarrow$) decreases as illuminance increases (incandescent bulb at varying distances), reflecting improved thermal SNR under stronger illumination. For reference, bright sunlight reaches 111,000 lux, while overcast daylight is typically 1,000–2,000 lux.}
    \label{fig:si-MSE_vs_illuminance}
    \vspace{-0.1in}
\end{figure}

\subsection{Non-Equilibrium Thermal Conditions}
\autoref{fig:failure_cases} shows a failure case of a truck where a heat source (engine) disturbs the thermal equilibrium. However, our approach can tolerate small deviations from thermal equilibrium since it relies on relative intensities (ordinalities) rather than absolute intensities. 
Theoretically, thermal transient exponentially converges to equilibrium~\cite{joint_light_and_heat_transport}, consequently, the observed thermal contrast follows this trend. 
Empirically, this tolerance can be evaluated on the JoLHT-Video dataset since all scenes were captured with a strong nearby source switched on at the beginning of videos, introducing a disturbance more severe than most natural settings. 
Please note that this disturbance is more severe than typical natural variations.
We ran our method at 10, 20, and 30 seconds, achieving 72.3\%, 79.0\%, 89.4\%, respectively, of the steady-state performance at 60s.

\section{Additional Method Details}
For point-pair sampling, the first point is sampled uniformly over the image. The second point is sampled at a random angle in $[0, 2\pi]$ and a random distance in $[0, 0.2]$ times the image diagonal from the first. Points outside the image are mirrored back.

\section{Runtime Analysis}
The optimization generally converges within 5000 iterations with Double-DIP parameterization and 500 without it. 
The average runtime for 5000 vs.\ 500 iterations is \text{48.22 s} vs.\ \text{5.63 s}, evaluated on the data from~\cite{joint_light_and_heat_transport}. 
The measurements were obtained by running five parallel jobs on a single GeForce RTX 4090 GPU and normalizing the runtime per job. 
In challenging cases, increasing the number of iterations can further improve the results.

\section{Additional Qualitative Results}
We provide additional qualitative results on \dataset dataset in comparison with state-of-the-art baselines. 
Each case shows visible input with albedo estimations above and thermal with shading below.
Images are tonemapped / colormapped for visualization. 

\begin{figure*}[t]
    \centering
    \includegraphics[width=1\linewidth]{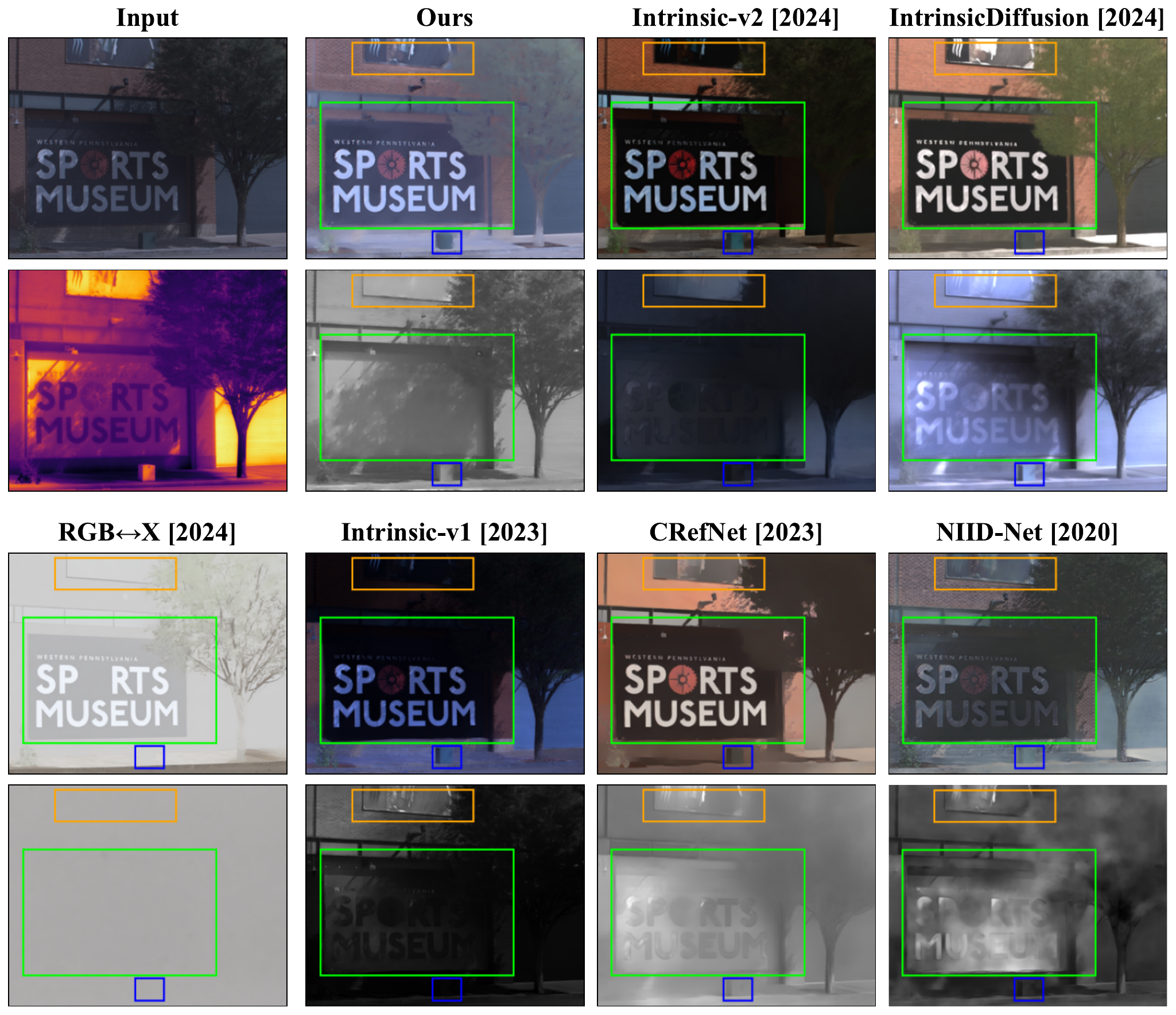}\\
    \vspace{-0.10in}
    \caption{
    Qualitative comparisons to state-of-the-art baselines. 
    Visible input with albedo estimations are shown above and thermal with shading below.
    }
\end{figure*}

\begin{figure*}[t]
    \centering
    \includegraphics[width=1\linewidth]{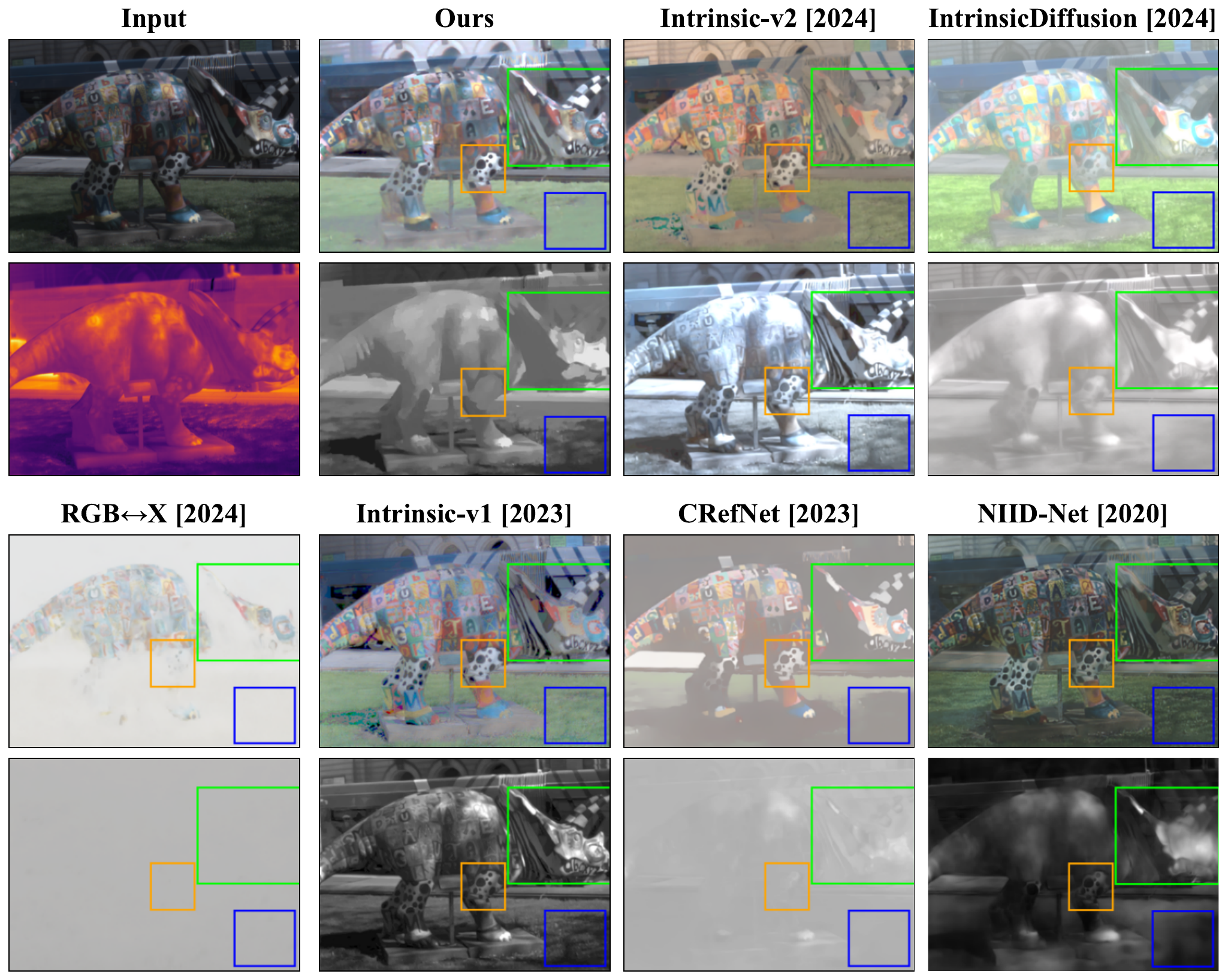}\\
    \vspace{-0.10in}
    \caption{
    Qualitative comparisons to state-of-the-art baselines.
    }
\end{figure*}

\begin{figure*}[t]
    \centering
    \includegraphics[width=1\linewidth]{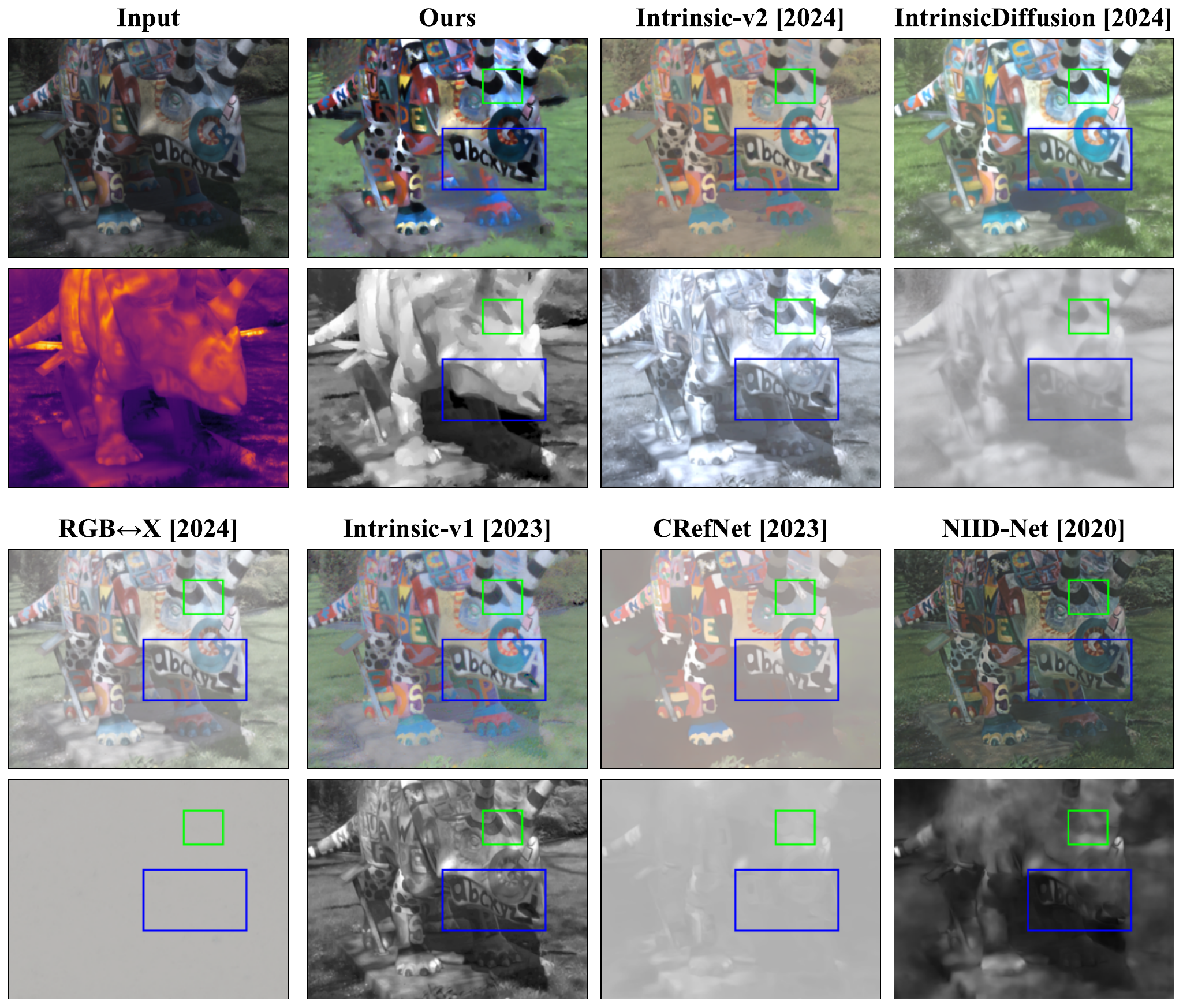}\\
    \vspace{-0.10in}
    \caption{
    Qualitative comparisons to state-of-the-art baselines.
    }
\end{figure*}

\begin{figure*}[t]
    \centering
    \includegraphics[width=1\linewidth]{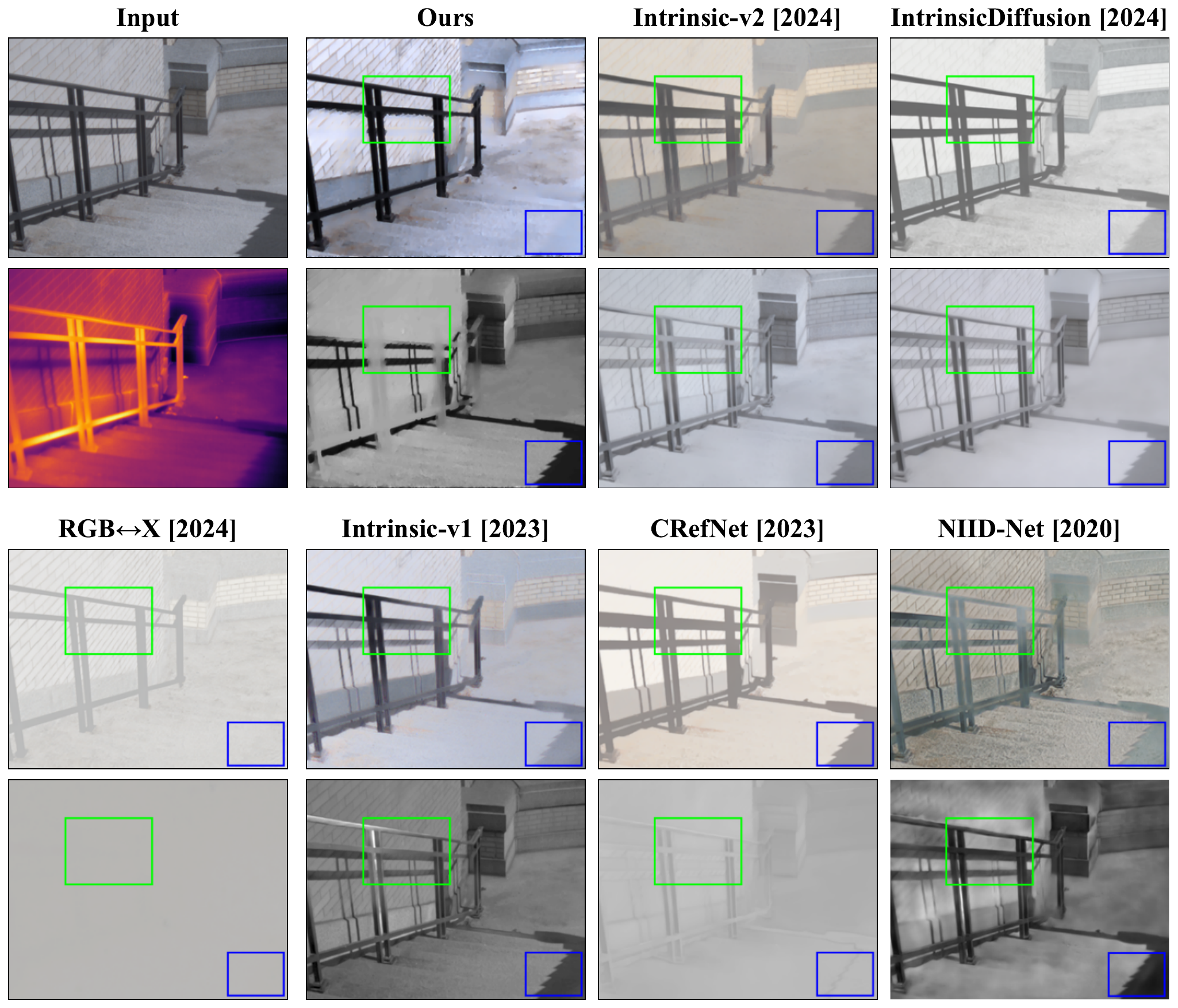}\\
    \vspace{-0.10in}
    \caption{
    Qualitative comparisons to state-of-the-art baselines.
    }
\end{figure*}

\begin{figure*}[t]
    \centering
    \includegraphics[width=1\linewidth]{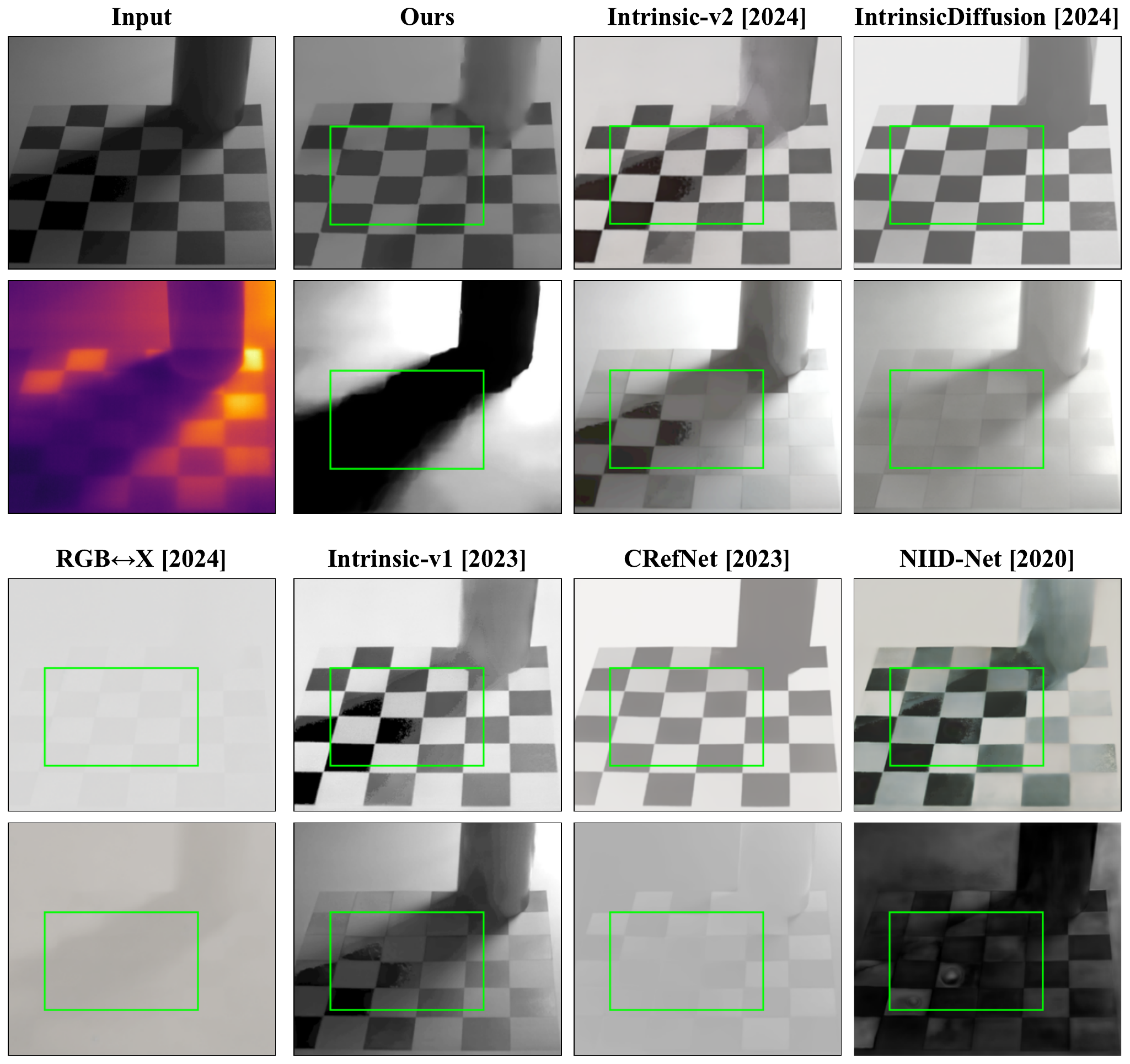}\\
    \vspace{-0.10in}
    \caption{
    Qualitative comparisons to state-of-the-art baselines.
    }
\end{figure*}

\begin{figure*}[t]
    \centering
    \includegraphics[width=1\linewidth]{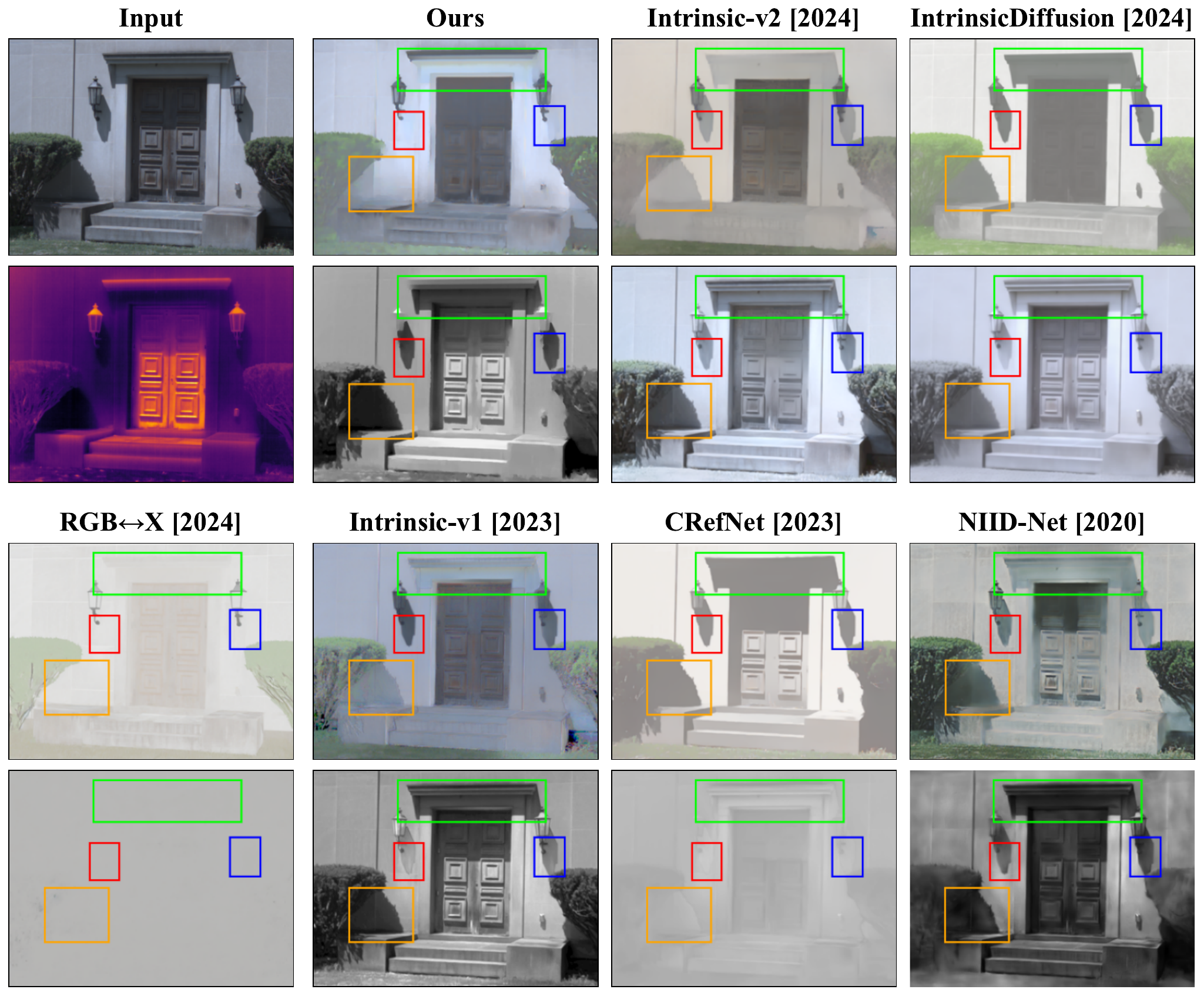}\\
    \vspace{-0.10in}
    \caption{
    Qualitative comparisons to state-of-the-art baselines.
    }
\end{figure*}

\begin{figure*}[t]
    \centering
    \includegraphics[width=1\linewidth]{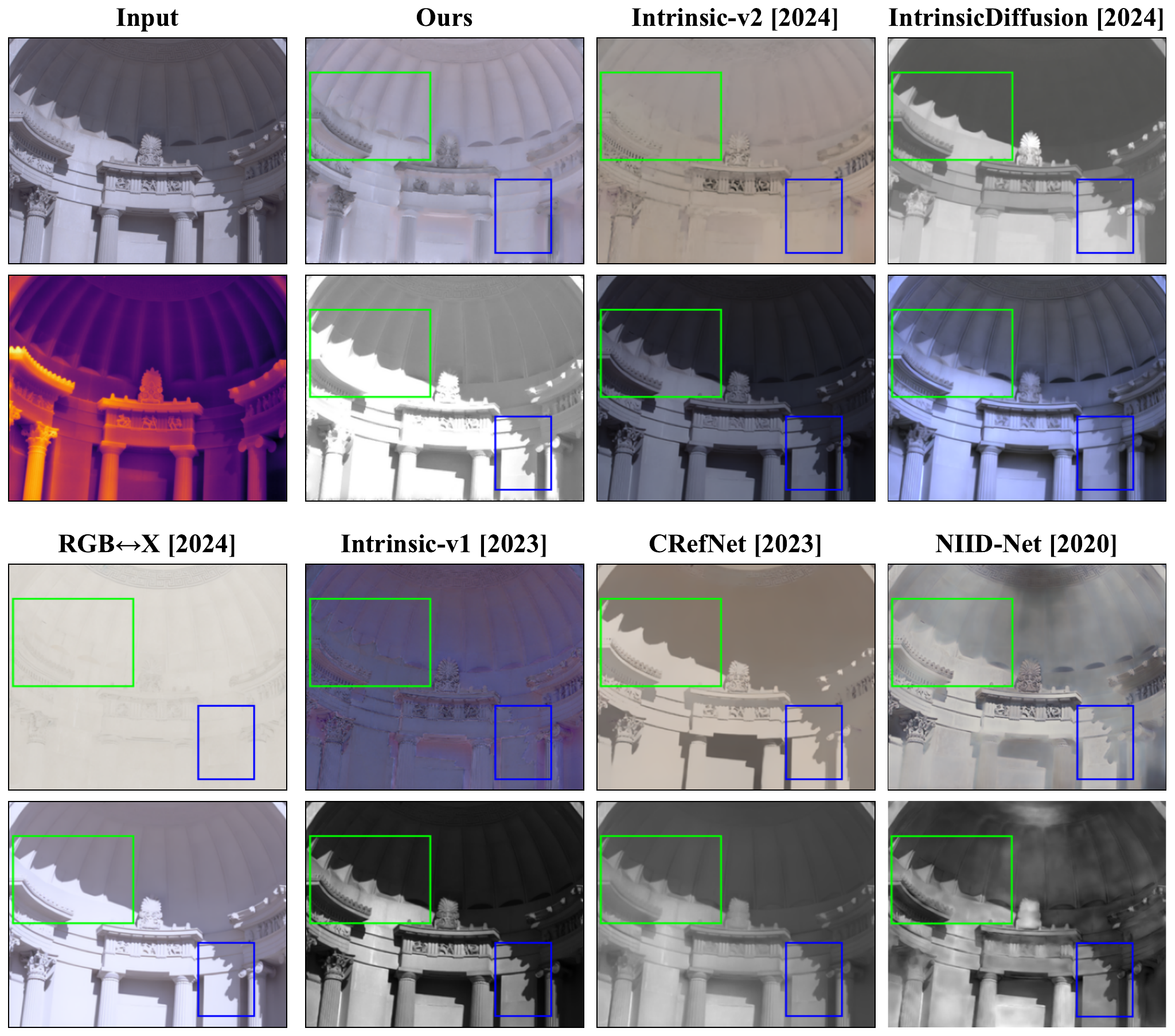}\\
    \vspace{-0.10in}
    \caption{
    Qualitative comparisons to state-of-the-art baselines.
    }
\end{figure*}

\begin{figure*}[t]
    \centering
    \includegraphics[width=1\linewidth]{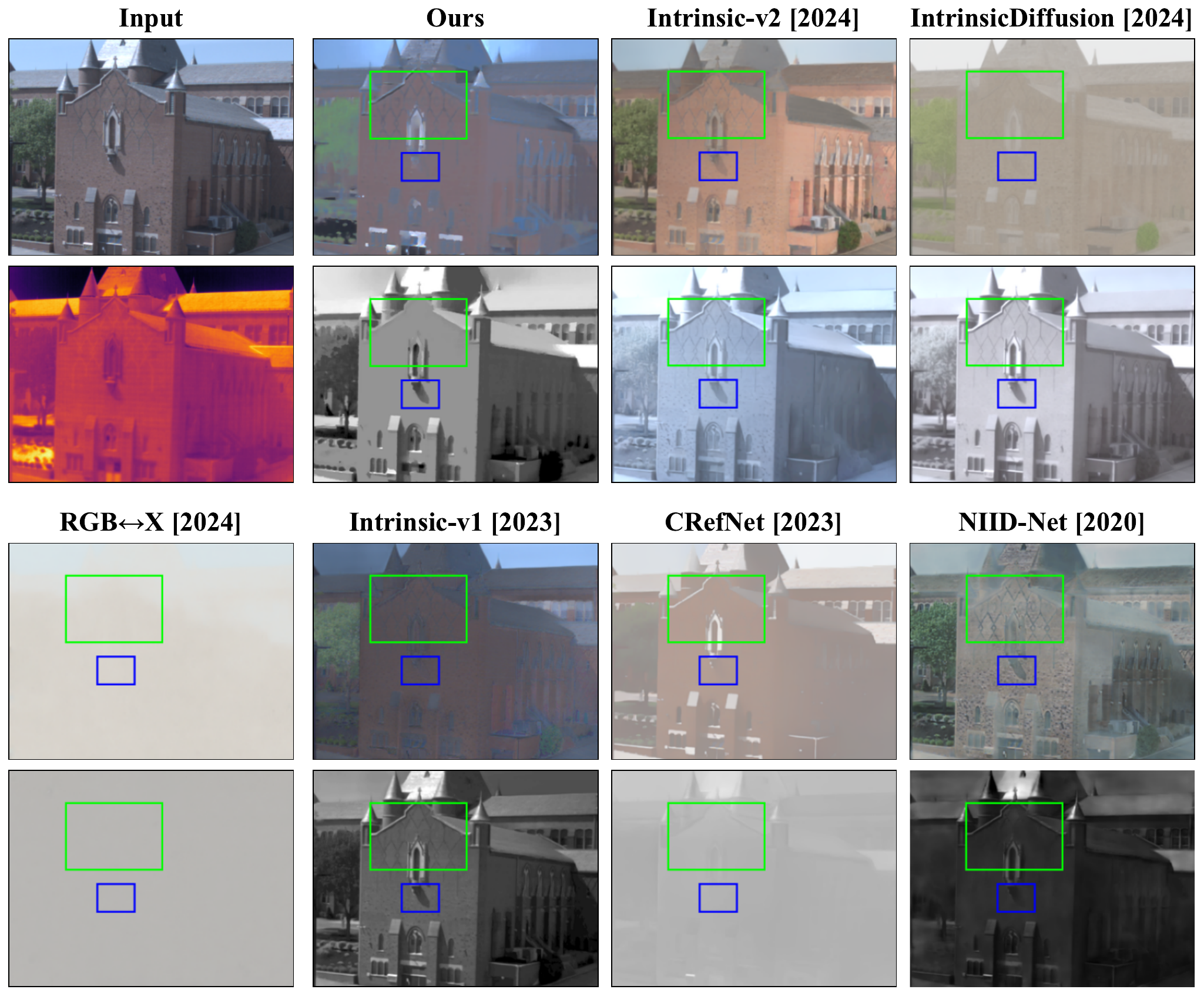}\\
    \vspace{-0.10in}
    \caption{
    Qualitative comparisons to state-of-the-art baselines.
    }
\end{figure*}

\begin{figure*}[t]
    \centering
    \includegraphics[width=1\linewidth]{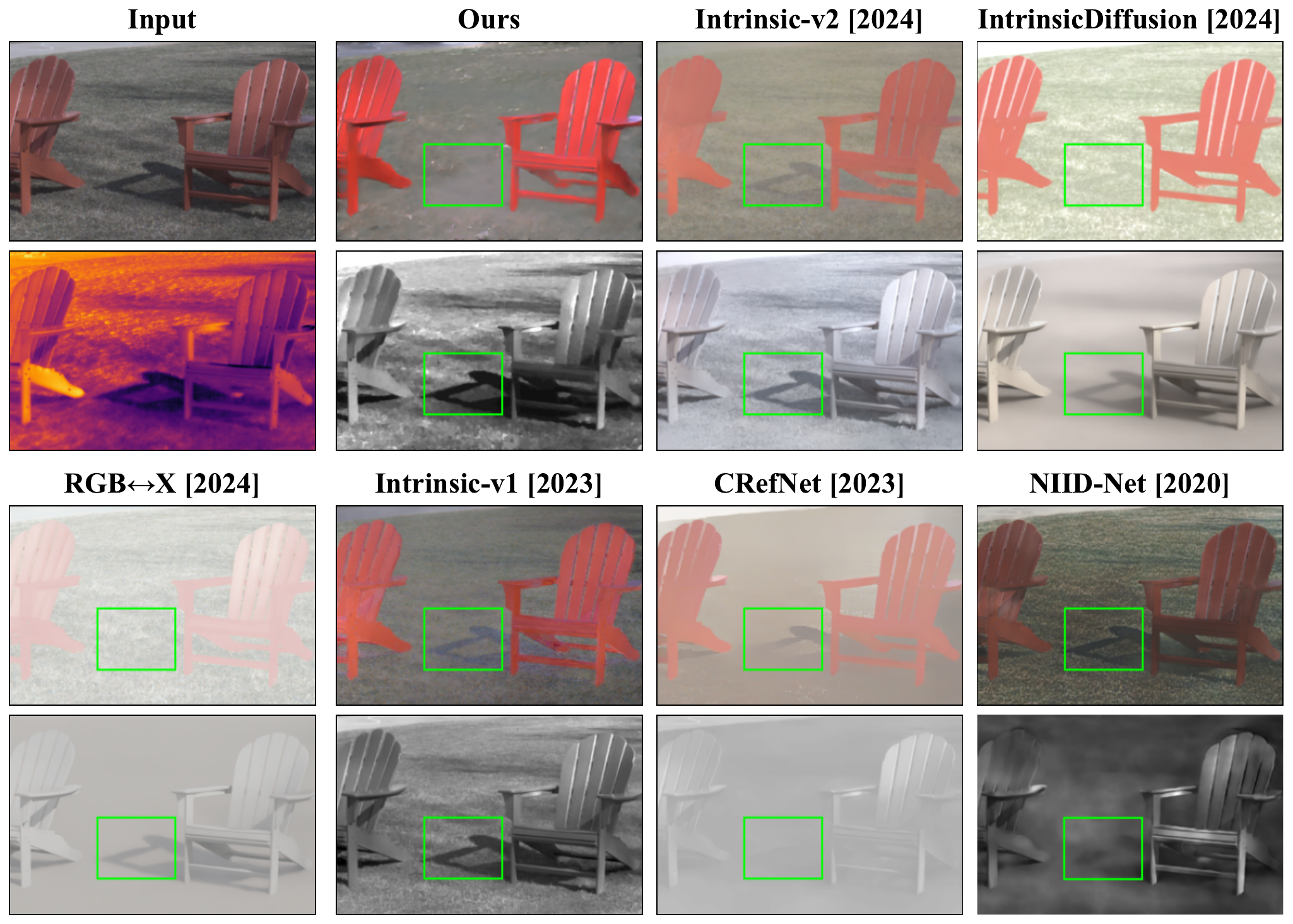}\\
    \vspace{-0.10in}
    \caption{
    Qualitative comparisons to state-of-the-art baselines.
    }
\end{figure*}

\begin{figure*}[t]
    \centering
    \includegraphics[width=1\linewidth]{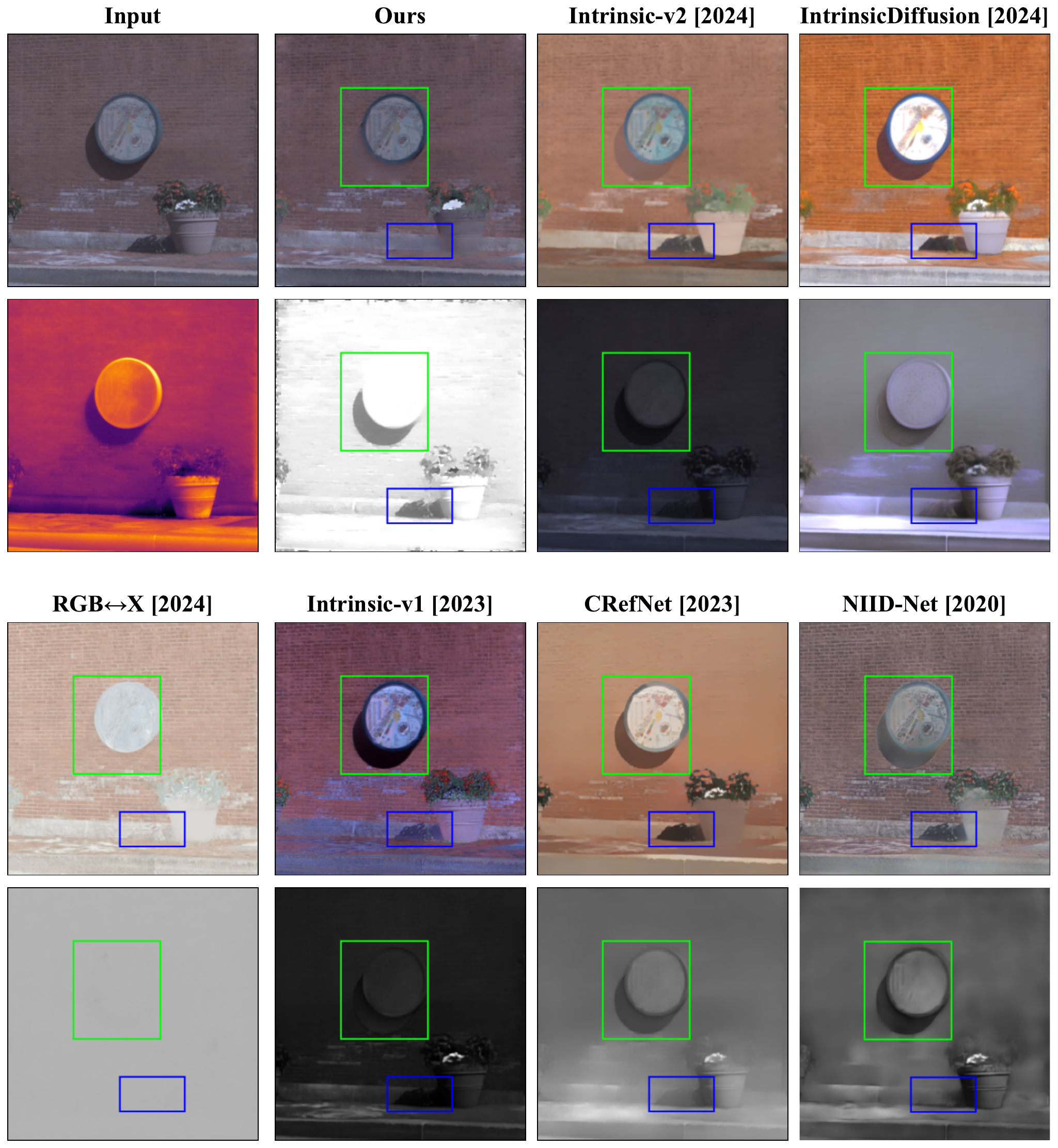}\\
    \vspace{-0.10in}
    \caption{
    Qualitative comparisons to state-of-the-art baselines.
    }
\end{figure*}

\begin{figure*}[t]
    \centering
    \includegraphics[width=1\linewidth]{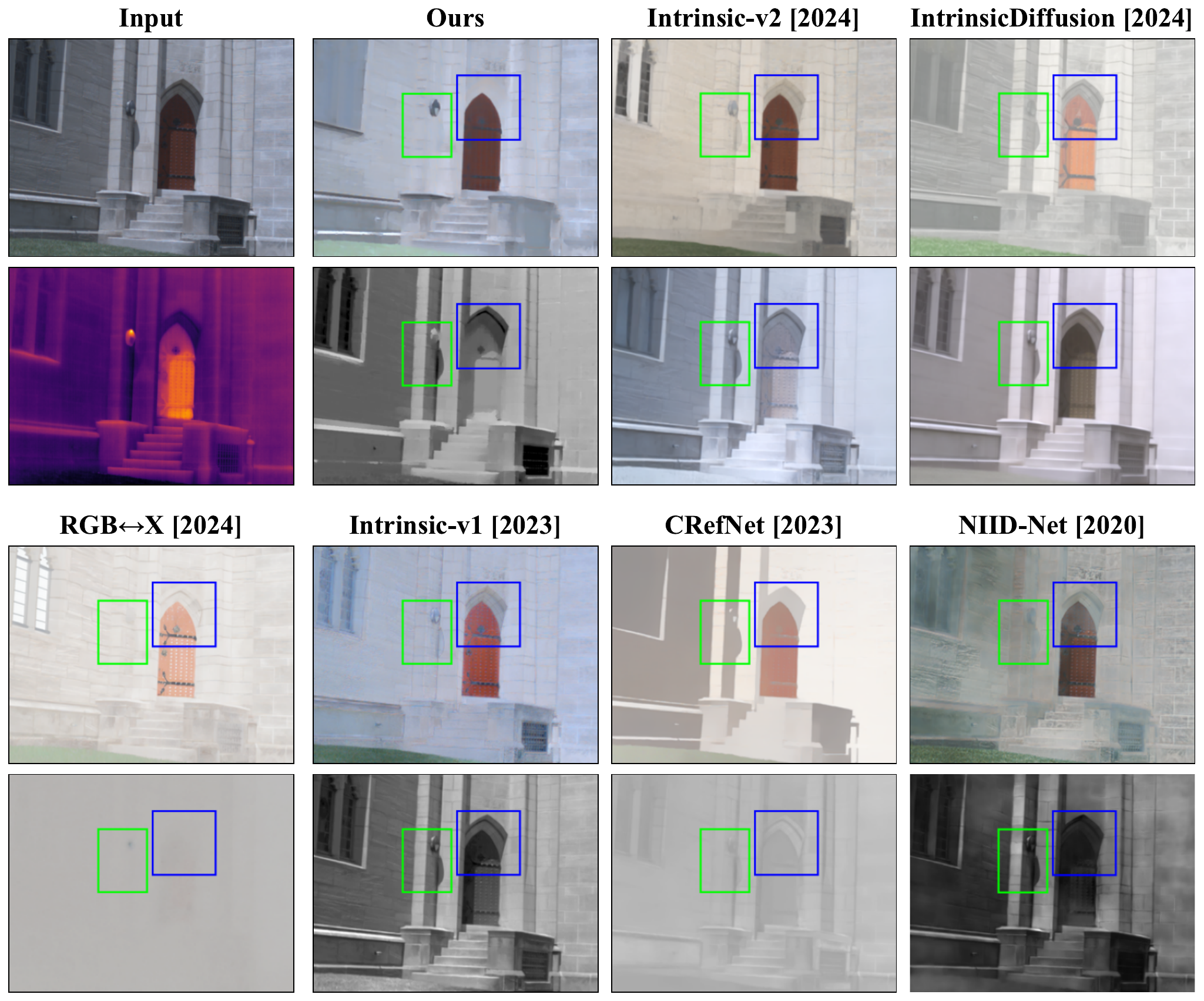}\\
    \vspace{-0.10in}
    \caption{
    Qualitative comparisons to state-of-the-art baselines.
    }
\end{figure*}

\begin{figure*}[t]
    \centering
    \includegraphics[width=1\linewidth]{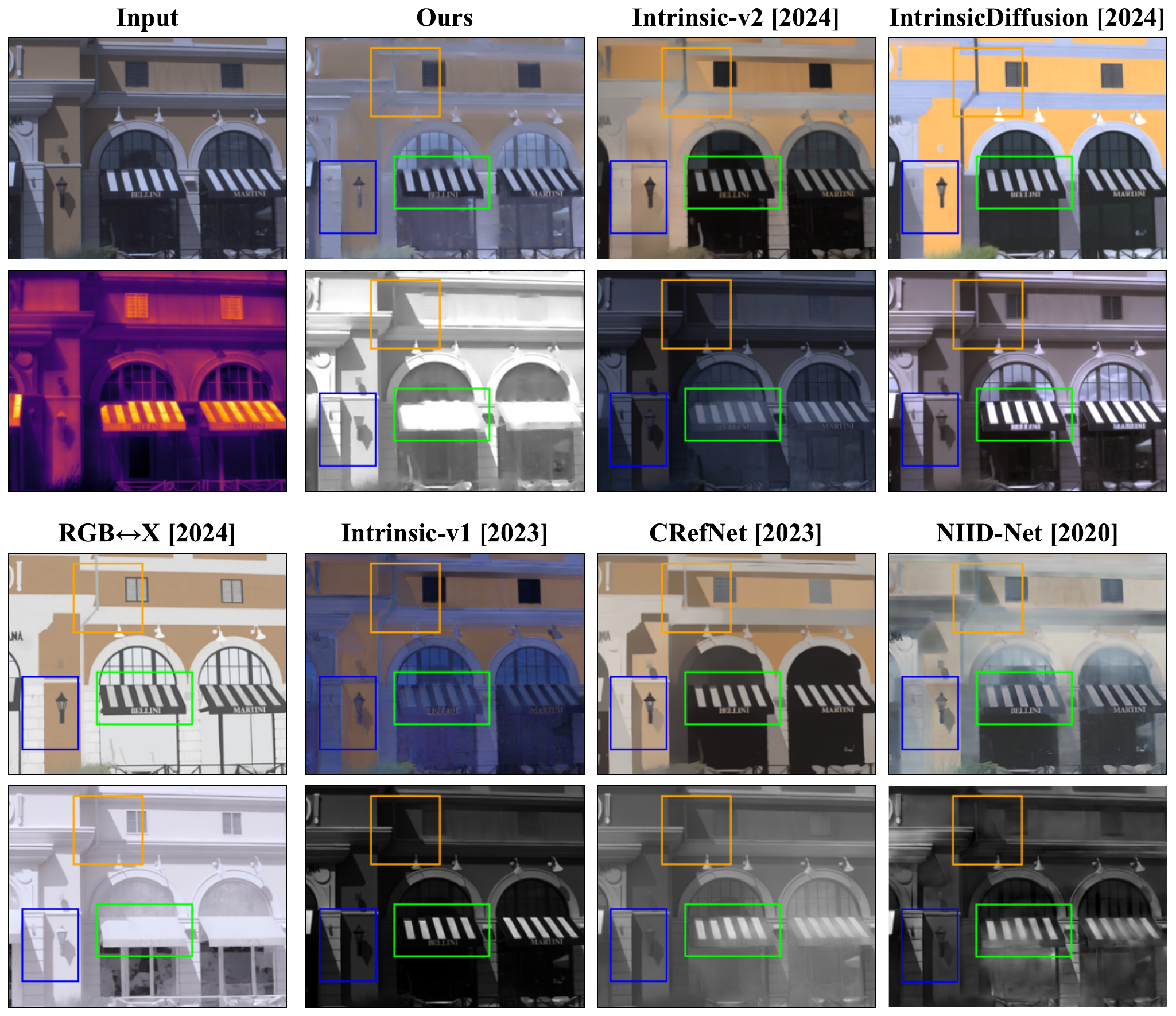}\\
    \vspace{-0.10in}
    \caption{
    Qualitative comparisons to state-of-the-art baselines.
    }
\end{figure*}

\begin{figure*}[t]
    \centering
    \includegraphics[width=1\linewidth]{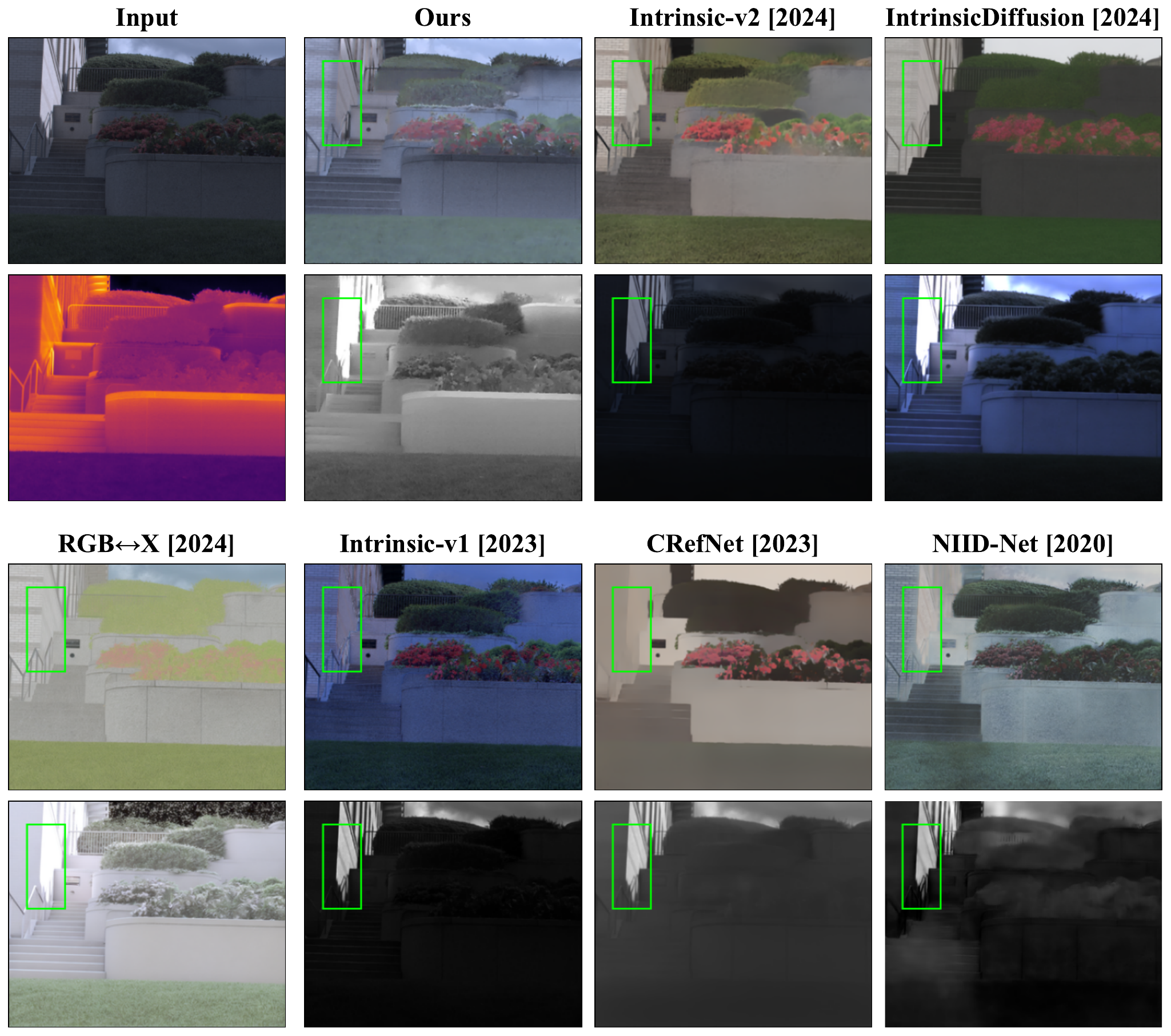}\\
    \vspace{-0.10in}
    \caption{
    Qualitative comparisons to state-of-the-art baselines.
    }
\end{figure*}

\begin{figure*}[t]
    \centering
    \includegraphics[width=0.7\linewidth]{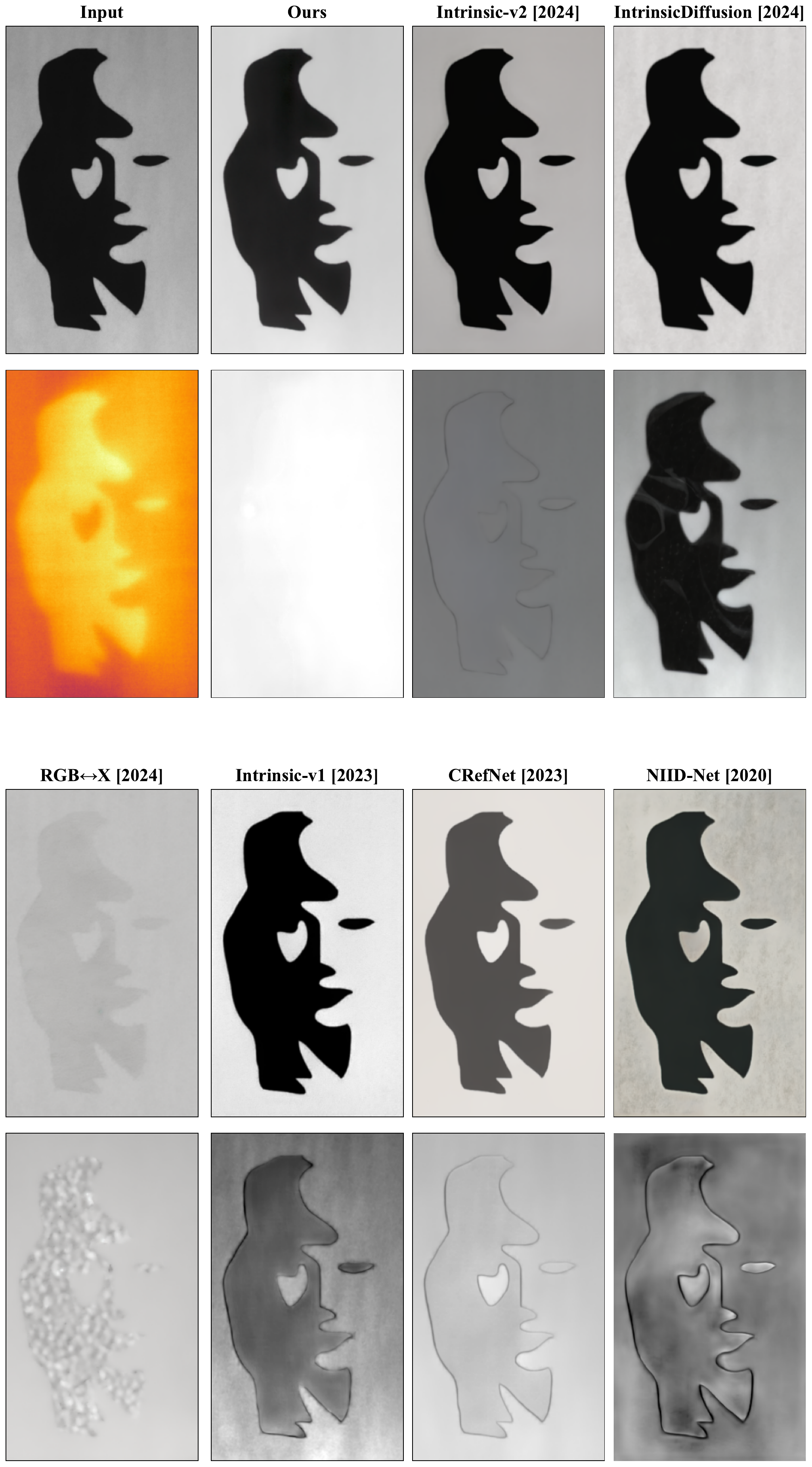}\\
    \vspace{-0.10in}
    \caption{
    Qualitative comparisons to state-of-the-art baselines on a printed image.
    }
\end{figure*}

\begin{figure*}[t]
    \centering
    \includegraphics[width=0.7\linewidth]{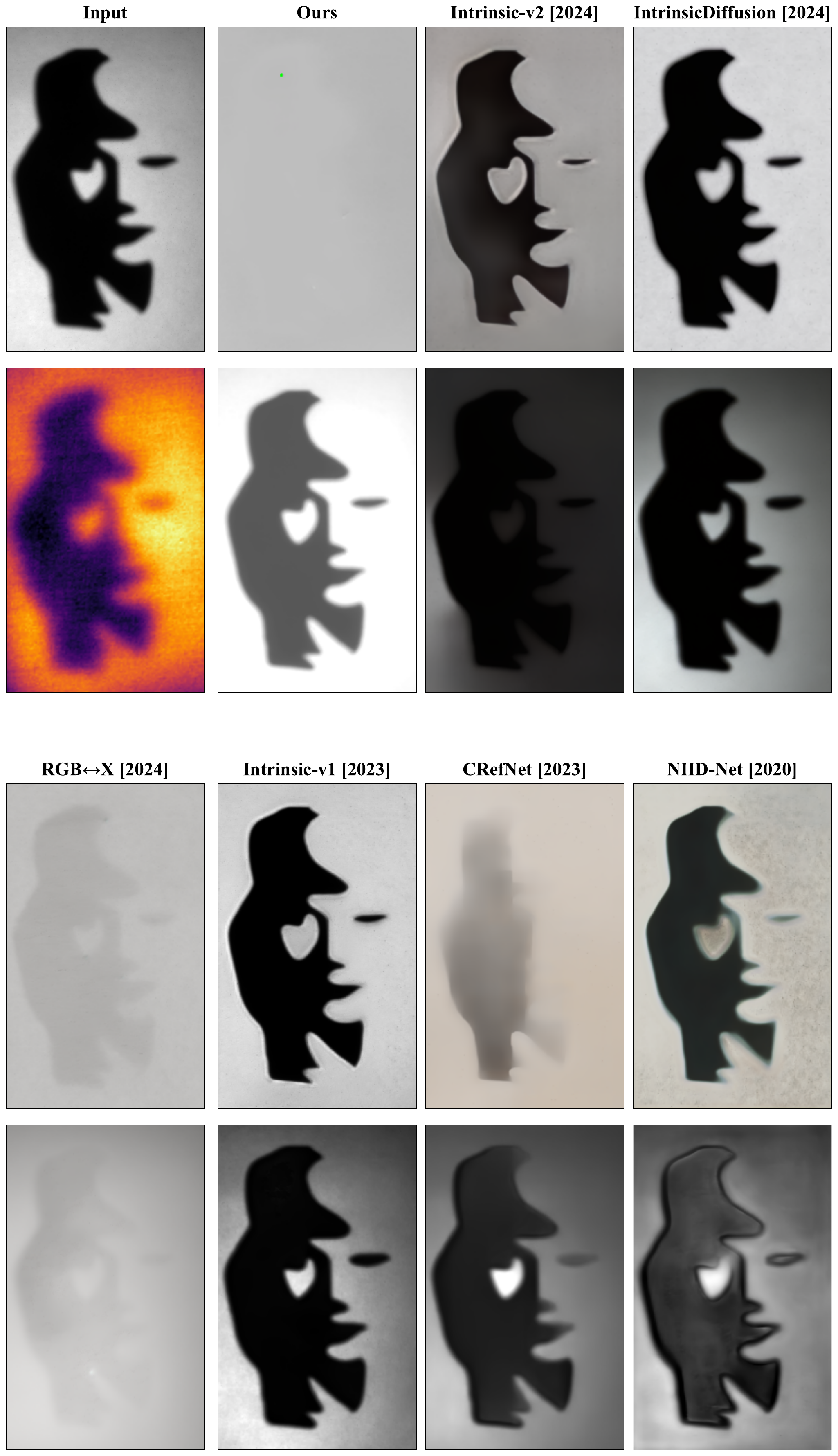}\\
    \vspace{-0.10in}
    \caption{
    Qualitative comparisons to state-of-the-art baselines on a projected image.
    }
\end{figure*}



\end{document}